\documentclass{article}
\PassOptionsToPackage{dvipsnames,table}{xcolor}
\usepackage{aibl_arxiv}

\usepackage{url}

\usepackage{pifont}%

\usepackage{booktabs}
\usepackage{graphicx}
\usepackage{multirow}
\usepackage{comment}
\usepackage{subcaption}

\usepackage{bbold}

\usepackage{amsfonts,amsmath,amssymb,amsthm}
\DeclareMathOperator{\KL}{KL}
\usepackage{cleveref}
\usepackage{tocloft,titletoc}

\usepackage{fontawesome5}
\usepackage[dvipsnames]{xcolor}
\usepackage{algorithm}
\usepackage{wrapfig}
\usepackage{booktabs}
\usepackage{multirow}
\usepackage{colortbl}

\usepackage[utf8]{inputenc}
\usepackage[T1]{fontenc}
\usepackage{microtype}
\usepackage{xcolor}
\usepackage{url}
\usepackage{hyperref}

\usepackage{amsmath,amssymb,amsfonts,amsthm,mathtools,bm}
\usepackage{nicefrac}

\usepackage{booktabs}
\usepackage{subcaption}
\usepackage{array}
\usepackage{tabularx}
\usepackage{makecell}
\usepackage{multirow}
\usepackage{longtable}
\usepackage{graphicx}

\usepackage{algorithm}
\usepackage{algpseudocode}
\usepackage{enumitem}
\usepackage{xspace}

\newcommand{\method}{\textsc{VISTA}\xspace}

\newcommand{\Rmax}{R_{\max}}

\newcommand{\ceilG}[1]{\lceil #1\rceil}

\newcolumntype{Y}{>{\raggedright\arraybackslash}X}
\newcolumntype{L}[1]{>{\raggedright\arraybackslash}p{#1}}
\newcolumntype{C}[1]{>{\centering\arraybackslash}p{#1}}

\theoremstyle{plain}
\newtheorem{theorem}{Theorem}[section]
\newtheorem{lemma}[theorem]{Lemma}

\theoremstyle{definition}
\newtheorem{definition}[theorem]{Definition}

\begin{document}

\aiblmaketitle
  {Value-Informed Schedule Trimming for Accelerated Sequential Monte Carlo Guidance in Discrete Diffusion}
  {Jaehyeon Kim$^{1,*}$, \ Ziseok Lee$^{1,*}$, \ Siyoung Choi$^{1,*}$, \ Kyungsu Kim$^{1,\dagger}$}
  {$^{1}$Seoul National University\\[0.15em]
   {\small $^{*}$equal contribution, $^{\dagger}$corresponding author: \texttt{\{kyskim\}@snu.ac.kr}}}
  {Reward-guided discrete diffusion enables test-time alignment with downstream objectives, but applying Sequential Monte Carlo (SMC) guidance at every denoising step is computationally expensive. Active guidance steps require additional reward evaluations and, for gradient-based proposals, value-gradient computations. Sparse guidance reduces this cost by guiding only a subset of timesteps, but commonly used schedules are fixed heuristics whose effectiveness can vary substantially across tasks. We propose VISTA, a training-free method that automatically selects guided timesteps under a fixed budget. VISTA formulates sparse guidance as approximation of the ideal full-guidance probability path and derives an estimable upper bound on \(\ell_1\) path deviation. Using value statistics collected from a small number of full-guidance warmup generations, VISTA optimizes the resulting empirical surrogate exactly via dynamic programming. The warmup generations can be retained as model outputs, and after schedule selection VISTA has the same per-generation guidance complexity as other sparse schedules with the same budget. Across text, biological sequence, and image generation tasks, VISTA discovers task-dependent schedules and outperforms hand-designed sparse schedules in most tested settings, while approaching or exceeding full-guidance performance on several task metrics. In toxic-text generation, for example, VISTA attains a 0.99 unique-toxic ratio compared with 1.00 for full guidance while reducing post-selection per-generation latency from 204.23 s to 38.16 s, a 5.35$\times$ speedup.}
  {
   \aiblresource{Lab}{https://aibl.snu.ac.kr}
   }

\section{Introduction}

Discrete diffusion models \citep{austin2021structured} have emerged as a powerful framework for generating structured objects such as text, biological sequences, and discrete images. In these domains, practitioners often prioritize not only the likelihood under a pretrained model but also alignment with external objectives, such as safety in language, biological activity in DNA, or constraint satisfaction. Consequently, a primary goal is to sample from a reward-tilted distribution that steers the pretrained model toward higher expected rewards while preserving its generative capabilities.

Several reward-tilted guidance techniques for discrete domains have recently been introduced \citep{nisonoff2024unlocking,schiff2024simple}. However, these independent sampling methods---which generate samples along separate trajectories---often rely on derivative-based approximations that can introduce significant error and bias \citep{lee2025debiasing}. To address this limitation, Sequential Monte Carlo (SMC) methods have been proposed. SMC begins by generating multiple candidate particles using a proposal policy. To align the particles with the reward-tilted distribution, it employs a resampling mechanism: particles with high value functions are duplicated, while those with low values are discarded (see Figure \ref{fig:FKillustrate}-(b)). Under the asymptotic conditions used in the SMC analysis, this framework recovers the exact reward-tilted distribution at guided timesteps
\citep{lee2025debiasing,ou2025inference,naesseth2019elements,feng2024step,kim2025alignment,luo2026self}.

Despite their theoretical elegance, the computational cost of SMC guidance remains a bottleneck in practical applications. Each guided denoising step requires not only a pretrained diffusion transition but also an evaluation by an external reward model; moreover, gradient-based proposals often necessitate additional backward passes. Applying this correction at every timestep $t$ can dominate the total wall-clock time. Consequently, sparse guidance---where only a subset of $T' < T$ timesteps is guided---provides a natural way to improve the efficiency of inference-time alignment, where $T$ denotes the total number of diffusion steps. This approach can reduce the total reward and gradient computation to a fraction of $T'/T$ (see Table~\ref{tab:cost} for a detailed breakdown).

However, a central question remains under-explored: which steps should be guided?
Existing sparse guidance schedules are typically heuristic-based and fixed, for example by restricting guidance to a prescribed interval \citep{kynkaanniemi2024applying,hoogeboom2024simpler,yu2024representation} or applying it at uniformly spaced timesteps \citep{singhal2025general, askari-hemmat2024feedbackguided, kim2025sampleefficient}.
Such fixed schedules can be suboptimal because the effective timing of guidance depends on the task and reward landscape.
As we show empirically, some tasks favor guidance concentrated in the early denoising stages, whereas others benefit from combining early guidance with corrections distributed throughout the reverse process.
Ultimately, a schedule that is effective for one reward landscape can be suboptimal for another.

To mitigate this issue, we propose VISTA-SMC (Value-Informed Schedule Trimming for Accelerated Sequential Monte Carlo), a training-free, plug-and-play procedure for identifying a budgeted schedule of guided steps. By minimizing a tractable surrogate for the discrepancy between the full-guidance trajectory and the reduced-guidance path, VISTA characterizes the guidance dynamics of a given task and identifies a highly effective schedule tailored to its specific reward landscape.

Specifically, VISTA addresses sparse SMC guidance as a full-path approximation problem: it chooses the guided-step set (i.e., $\mathcal{G}$) so that the reduced-guidance path stays close to the ideal full-guidance path (i.e., $\{\pi_t^*\}_{t=0}^{T}$). Directly optimizing this gap is infeasible because the ground-truth full-guidance marginals are not available as explicit distributions over the discrete state space. Our main observation is that the gap admits a tractable upper bound expressed through value expectations along the full-guidance path. These expectations can be estimated from a small number of warmup full-guidance runs; the resulting empirical budgeted objective is then solved exactly by dynamic programming. The selected schedule is reused for all later generations with the same model, reward, budget, and proposal family, so the warmup cost is amortized and the per-generation comparison to sparse heuristics is fair.

Our contributions are ordered by conceptual importance:
\begin{enumerate}
    \item \textbf{Automatic budgeted guidance scheduling.}
    We introduce a training-free scheduler that automatically selects a task-specific subset of SMC guidance timesteps under a fixed budget, using full-guidance value statistics instead of a predefined temporal pattern.

    \item \textbf{Full-guidance gap formulation.}
    We view sparse SMC guidance as approximating the ideal full-guidance probability path, not only the final marginal. This is crucial: final-only correction may be asymptotically valid yet can lead to diversity collapse with finite particles.

    \item \textbf{Tractable gap minimization without ground-truth marginals.}
    To our knowledge, VISTA is the first sparse-SMC scheduler to target the gap to full-guidance marginals without accessing those intractable marginals. We prove an estimable upper bound on $\ell_1$ path deviation and optimize the resulting surrogate under a fixed budget.

    \item \textbf{Practical schedule construction.}
    VISTA uses a few initial full-guidance generations to estimate a task-specific schedule. These generations are retained as generated outputs rather than discarded preprocessing. Afterward, VISTA has the same per-generation complexity as heuristic sparse schedules with the same budget, although heuristics can operate sparsely from the first generation.

    \item \textbf{Task-specific empirical schedules.}
    On Binary MNIST, toxic text, and DNA sequence generation, VISTA finds reward-dependent schedules and outperforms hand-designed sparse baselines in most tested settings at the same budget. It also yields large speedups over full guidance, including a $5.35\times$ post-selection speedup in the toxic-text SMC-grad experiment at $T'=15$, while achieving a unique-toxic ratio of $0.99$ compared with $1.00$ under full guidance.
\end{enumerate}

\section{Related Work}

\paragraph{Discrete and Masked Diffusion.} Diffusion models define a forward process that transforms a data distribution into a tractable prior and learn reverse transitions for generation. While continuous diffusion models typically use Gaussian perturbations and score-based reverse dynamics~\citep{ho2020denoising, song2020score}, discrete diffusion models define Markov transitions over categorical or masked states~\citep{austin2021structured, sahoo2024simple}. Masked diffusion language models instantiate this idea by using a special [MASK] token as the high-noise state and iteratively unmasking or denoising tokens. Our work focuses on inference-time reward guidance for such reverse processes, with a particular emphasis on addressing the challenges of discrete state spaces.

\paragraph{Sequential Monte Carlo Methods.} SMC samplers propagate particles through a proposal policy and use resampling to correct bias, effectively aligning the particles with the target reward-tilted distribution~\citep{del2006sequential, naesseth2019elements, chopin2020introduction}. Recent work has adapted SMC to diffusion inference for tasks such as continuous and discrete diffusion steering, test-time alignment, inverse problems, and tree-structured search~\citep{ou2025inference, skreta2025feynman, hasan2026discrete, achituve2025inverse, wu2025reverse, jain2025diffusion, guo2025training, chanchudiscrete}. While these methods primarily focus on \emph{how} particles are weighted or proposed, \method is complementary: it automatically identifies the most critical timesteps for guidance, determining \emph{when} these expensive corrections should be executed to maximize overall efficiency.

\paragraph{Inference-time Acceleration of Diffusion Guidance.}
Prior work has reduced the cost of diffusion guidance using hand-designed temporal schedules,
including restricting guidance to a prescribed interval
\citep{kynkaanniemi2024applying,hoogeboom2024simpler,yu2024representation}
or uniformly spaced timesteps
\citep{singhal2025general, askari-hemmat2024feedbackguided, kim2025sampleefficient}.
Other approaches reduce guidance overhead through adaptive rules, caching, or multirate integration
\citep{castillo2025adaptive,lv2024fastercache,lee2025tortoise}.
These methods demonstrate that guidance need not be applied at every denoising step, but their schedules are either manually specified or developed primarily for continuous-image guidance and network-evaluation reuse. 
In contrast, \method is specifically formulated for reward-guided SMC in discrete state spaces and provides a theoretical bound on the probability-path deviation. Table~\ref{tab:difference} summarizes these key distinctions.

\vspace{-0.3cm}
\begin{table}[h]
    \centering
    \small
    \caption{Comparison of \method with existing guidance acceleration methods. Discrete: Designed for discrete state spaces; Reward: Handles external reward functions; Theoretical Guarantee: Provides a rigorous theoretical justification for the acceleration schedule.}
    \vspace{0.3cm}
    \label{tab:difference}
    \begin{tabular}{lccc}
    \toprule
    \textbf{Method} & \textbf{Discrete} & \textbf{Reward} & \textbf{Theoretical Guarantee} \\
    \midrule
    Adaptive Guidance~\citep{castillo2025adaptive}      & \texttimes & \texttimes & \texttimes \\
    FasterCache~\citep{lv2024fastercache}        & \texttimes & \texttimes & \texttimes \\
    Tortoise-and-Hare~\citep{lee2025tortoise}          & \texttimes & \texttimes & \checkmark \\
    Guidance Interval~\citep{kynkaanniemi2024applying} & \texttimes & \texttimes & \texttimes \\ 
    \midrule
    \textbf{\method (Ours)}                    & \checkmark & \checkmark & \checkmark \\ 
    \bottomrule
    \end{tabular}
\end{table}

\vspace{-0.2cm}
\section{Preliminaries}

\subsection{Reference Reverse Process}

Let $\mathcal{X}$ be a discrete state space. A pretrained reverse diffusion model defines a reference path:
\begin{equation}
    p_T(x_T) \prod_{t=0}^{T-1} p_{t | t+1}(x_t | x_{t+1}),
\end{equation}
with one-step reverse kernels $p_{t|t+1}$ and marginals $p_t$. For $0 \leq s \leq t \leq T$, we denote by $p_{s|t}$ the composed transition from $x_t$ to $x_s$. In a masked diffusion model, $x_T$ is typically a fully masked sequence, and $x_0$ is a clean sequence.

\subsection{Reward-Tilted Full-Guidance Path}

Given a reward function $r(x_0)$ for a sample $x_0$, our objective is to maximize the expected reward while regularizing the divergence from the pretrained model's prior. Specifically, we aim to find an optimal policy $\{\pi_t\}_{t=T}^0$ that maximizes:

\begin{equation}
\mathbb{E}_{\{\pi_t\}_{t=T}^0} \left[ r(x_0) - \alpha \sum_{t=0}^{T-1} \text{KL}(\pi_{t | t+1}(\cdot | x_{t+1}) \parallel p_{t | t+1}(\cdot | x_{t+1})) - \alpha \text{KL}(\pi_T(\cdot) \parallel p_T(\cdot)) \right],
\end{equation}

where $\alpha > 0$ is a temperature parameter controlling the trade-off between reward maximization and prior preservation~\citep{fan2023dpok, uehara2024understanding, uehara2024bridging}. The closed-form solution to this optimization problem is:

\begin{align}
\pi_{t|t+1}^*(x_t | x_{t+1}) &= \frac{V_t(x_t)}{V_{t+1}(x_{t+1})} p_{t|t+1}(x_t | x_{t+1}), \quad t \in \{0, \dots, T-1\} \\
\pi_t^*(x_t) &= \frac{1}{Z} V_t(x_t) p_t(x_t),
\end{align}

where $Z := \mathbb{E}_{x_0 \sim p_0}[R(x_0)]$ is the partition function. Here, $R(x_0) := \exp(r(x_0)/\alpha)$ represents the exponentiated reward, and $V_t(x_t)$ is the exponentiated value function defined as:
\begin{equation}
    V_t(x_t) := \mathbb{E}_{x_0 \sim p_{0|t}(\cdot | x_t)} \left[ R(x_0) \right].
\end{equation}
Although the value function $v_t(x_t) := \alpha \log V_t(x_t)$ is standard in the literature~\citep{uehara2024bridging, uehara2025inference}, the exponentiated notation $V_t$ simplifies our derivation.

For a detailed derivation, we refer the reader to~\citep{uehara2024bridging} or our Appendix \ref{appendixbackground}.

\begin{definition}
We call $\pi_t^*$ the \textbf{full-guidance path}. It is the gold-standard path that \method aims to approximate.
\end{definition}

\subsection{SMC Proposal Families}

\begin{algorithm}
\caption{SMC guidance}\label{alg:smcshort}
\begin{algorithmic}[1]
\State $x_T^{(i)} \sim \pi_T^*(x_T) \quad (i=1, 2, \cdots, N)$
\For{$t = T-1, T-2, \cdots, 0$}
    \State Propagation: $x_t^{(i)} \overset{\text{iid}}{\sim} q_{t|t+1}(x_t | x_{t+1}^{(i)}) \quad (i=1, 2, \cdots, N)$
    \State Resampling: Duplicate particles with high $V_t$ and discard those with low $V_t$ to ensure $x_t^{(i)} \sim \pi_t^*(x_t)$
\EndFor
\end{algorithmic}
\end{algorithm}

SMC guidance propagates particles through a proposal policy $q_{t|t+1}$ and utilizes the resampling to align the particle distribution with the target path $\pi_t^*$ (Algorithm \ref{alg:smcshort}; see Algorithm \ref{alg:full_smc} for details). Generally, the closer $q_{t|t+1}$ is to $\pi_{t|t+1}^*$, the higher the sampling efficiency. We evaluate \method on two proposal families established in recent discrete SMC literature~\citep{ou2025inference}:

\begin{itemize}
\item \textbf{SMC-base} uses the pretrained transition: $q_{t | t+1}(x_t | x_{t+1}) = p_{t | t+1}(x_t | x_{t+1})$.
\item \textbf{SMC-grad} uses a first-order Taylor approximation of the value function to steer the proposal, yielding $q_{t|t+1}(x_t | x_{t+1}) \propto p_{t|t+1}(x_t | x_{t+1}) \exp(\frac{1}{\alpha} \nabla v_{t+1}(x_{t+1}) \cdot x_t)$.
\end{itemize}

\method is proposal-agnostic; it does not modify the underlying proposal family. Instead, it strategically selects which timesteps should incur the computational overhead of value estimation and resampling to maximize efficiency.

\section{\method: Value-Informed Schedule Trimming}

In this section, we introduce Reduced SMC Guidance, a framework designed to accelerate the inference of guided diffusion models by strategically skipping guidance steps.

\subsection{Sparse guidance as full-path approximation}

Let $\mathcal{T}=\{0, 1, \cdots, T-1\}$ denote the set of reverse timesteps. A sparse schedule is a subset $\mathcal{G}\subset \mathcal{T}$ with budget $|\mathcal{G}|=T'$. Under this framework, reward guidance is applied if and only if $t \in \mathcal{G}$, as detailed in Algorithm \ref{reducedguidance}.

\begin{algorithm}
\caption{Diffusion Guidance Reduction} \label{reducedguidance}
\begin{algorithmic}[1]
\State $x_T^{(i)} \sim \pi_T^*(x_T) \quad (i=1, 2, \cdots, N)$
\For{$t = T-1, T-2, \cdots, 0$}
    \State $\text{Guidance is ON} \gets t \in \mathcal{G}$
    \If{$\text{Guidance is ON}$}
        \State Propagation: $x_t^{(i)} \overset{\text{iid}}{\sim} q_{t|t+1}(x_t | x_{t+1}^{(i)}) \quad (i=1, 2, \cdots, N)$
    \State Resampling: Adjust particles $\{x_t^{(i)}\}$ to ensure $x_t^{(i)} \sim \pi_t^*(x_t)$
    \Else
        \State Propagation: $x_t^{(i)} \overset{\text{iid}}{\sim} p_{t|t+1}(x_t | x_{t+1}^{(i)}) \quad (i=1, 2, \cdots, N)$ \label{guideoff}
    \EndIf    
\EndFor
\end{algorithmic}
\end{algorithm}

For any $t \in \mathcal{T}$, we define the most recent active step $\ceilG{t}$ as:
\begin{equation}
\ceilG{t} := \min \{s \in \mathcal{G} \cup \{T\} \mid s \geq t\}.
\end{equation}

If $t\in \mathcal{G}$, then $\ceilG{t}=t$ and the sampler performs a guided SMC correction. If $t\notin \mathcal{G}$, it propagates from $\ceilG{t}$ to $t$ using unguided reference transitions. This observation leads to the following definition:
\begin{definition}
    For a sparse schedule $\mathcal{G}$, the \textbf{reduced-guidance path} is
    \begin{equation}
        \pi_t^{\mathcal{G}}(x_t) := \sum_{x_{\ceilG{t} \in \mathcal{X}}} \pi_{\ceilG{t}}^*(x_{\ceilG{t}})p_{t | \ceilG{t}}(x_t | x_{\ceilG{t}}).
    \end{equation}
\end{definition}

A naive objective would be to find $\mathcal{G}$ such that only the final marginal $\pi_0^{\mathcal{G}}$ matches the target $\pi_0^*$. Theoretically, if the number of particles $N$ is infinite, even a minimal schedule like $\mathcal{G} = \{0\}$ is sufficient to recover the exact reward-tilted distribution at the final step. However, this ignores a critical practical challenge: in the finite-$N$ regime, delaying guidance until the end causes the sampler to postpone all reward information. This leads to diversity collapse and poor reward (see Fig. \ref{fig:FKillustrate}). To mitigate this degradation in practical sampling, we instead aim to minimize the deviation between the full-guidance and reduced-guidance paths across all timesteps.

\begin{figure}
    \centering
    \includegraphics[trim={0 1cm 0 0}, clip, width=0.95\linewidth]{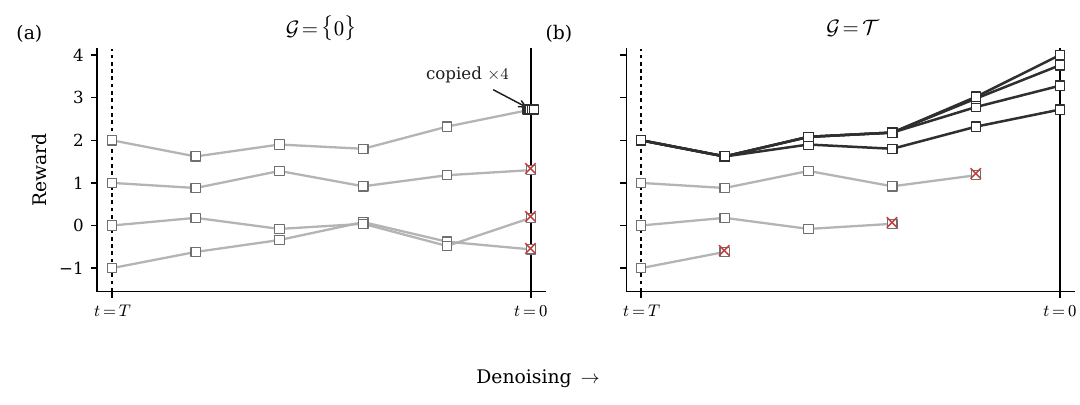}
    \caption{Impact of guidance schedules on particle distributions. A minimal schedule (e.g., $\mathcal{G}=\{0\}$) is asymptotically correct as $N \to \infty$, but with finite particles, it leads to diversity collapse through the duplication of a few terminal particles. In contrast, full guidance ($\mathcal{G}=\mathcal{T}$) maintains intermediate marginals close to the reward-tilted path, ensuring both high reward and sample diversity.}
    \label{fig:FKillustrate}
\end{figure}

\begin{definition}
    The \textbf{Total Deviation} for a schedule $\mathcal{G}$ is defined as the average $\ell_1$ distance between the full-guidance and reduced-guidance paths:
    \begin{equation}
        \mathbf{TD}(\mathcal{G}) := \frac{1}{T}\sum_{t=0}^{T-1} \|\pi_t^* - \pi_t^{\mathcal{G}}\|_1 = \frac{1}{T}\sum_{t=0}^{T-1} \sum_{x_t \in \mathcal{X}} |\pi_t^*(x_t) - \pi_t^{\mathcal{G}}(x_t)|.
    \end{equation}
\end{definition}

This definition formalizes the core objective of \method: schedule selection is transformed from a heuristic choice of intervals into a principled, budgeted optimization to minimize the gap between the sparse sampling path and the ideal full-guidance probability path.

\subsection{A tractable upper bound}

A major obstacle in directly optimizing $\mathbf{TD}(\mathcal{G})$ is the intractability of the full-guidance path $\pi_t^*$. The following theorem provides a computable surrogate based on value expectations that can be estimated from a small number of warmup runs.

\begin{theorem}\label{UTDthm}
    Assume $0< R(x_0) \leq \Rmax < \infty$. For any schedule $\mathcal{G}\subset \mathcal{T}$,
    \begin{equation}
        \mathbf{TD}(\mathcal{G}) \leq \mathbf{UTD}(\mathcal{G}) := \frac{2}{T}\sum_{t\in \mathcal{T}\setminus\mathcal{G}} [1 - \frac{1}{R_{\text{max}}}\mathbb{E}_{x_{\ceilG{t}}\sim \pi_{\ceilG{t}}^*}[V_{\ceilG{t}}(x_{\ceilG{t}})]].
    \end{equation}
\end{theorem}
The proof of this inequality is provided in Section \ref{TDsec}.

Our goal is to identify a schedule $\mathcal{G}$ that minimizes $\mathbf{UTD}$. This approach—optimizing a tractable upper bound when the primary objective is intractable—is a well-established strategy in machine learning~\cite{blei2017variational, oord2018representation, nowozin2016f}.

For a fixed budget of guidance steps $T' = |\mathcal{G}|$, minimizing $\mathbf{UTD}$ is equivalent to maximizing the following objective:
\begin{equation}
S(\mathcal{G}) := \sum_{t \in \mathcal{T} \setminus \mathcal{G}} \mathbb{E}_{x_{\ceilG{t}} \sim \pi^*_{\ceilG{t}}} \left[ V_{\ceilG{t}}(x_{\ceilG{t}}) \right].
\end{equation}
Intuitively, if the state at the most recent active step $x_{\ceilG{t}}$ already exhibits a high expected reward, subsequent unguided propagation is less likely to deviate significantly from the optimal path.

Crucially, this optimization does not require knowledge of the full marginals $\pi_t^*(x_t)$, nor the constant $\Rmax$; it relies solely on the Monte Carlo estimation of $\mathbb{E}_{x_t \sim \pi^*_t} \left[ V_t(x_t) \right]$.

\subsection{Description of the Method}

\method estimates $\mathbb{E}_{x_t \sim \pi^*_t} \left[ V_t(x_t) \right]$ by running $M$ full-guidance warmup simulations. Let $\hat{V}_t^{(i, m)}$ be the value estimate of $i$-th particle at timestep $t$ during the $m$-th warmup run. We approximate the mean value as:
\begin{equation}\label{estimatevalue}
    \mathbb{E}_{x_t \sim \pi^*_t} \left[ V_t(x_t) \right] \approx \hat{V}_t := \frac{1}{MN} \sum_{m=1}^M \sum_{i=1}^N \hat{V}_t^{(i, m)}.
\end{equation}
The empirical schedule $\hat{\mathcal{G}}$ is then obtained by solving:
\begin{equation}\label{optG}
    \hat{\mathcal{G}} \in \arg\max_{\mathcal{G}\subset \mathcal{T}: |\mathcal{G}|=T'} \hat{S}(\mathcal{G}), \quad \hat{S}(\mathcal{G}) := \sum_{t \in \mathcal{T}\setminus \mathcal{G}} \hat{V}_{\ceilG{t}}
\end{equation}

The optimization in Eq. \eqref{optG} is solved exactly by dynamic programming in $O(T^2T')$ time (see Algorithm \ref{alg:dp}), which is negligible for the horizons used here. Once $\hat{\mathcal{G}}$ is selected, \method utilizes this sparse schedule for all subsequent generations to accelerate inference.

\subsection{Computational Time Analysis}\label{sec:costanalysis}

\method requires $M$ initial full-guidance generations for schedule
selection. Once the schedule is determined, it has the same per-generation
complexity as other sparse guidance methods under the same guidance budget.
Table~\ref{tab:cost} summarizes the resulting computational costs.

\begin{table}[t]
\centering
\small
\caption{Per-generation cost after schedule selection. Here, $p$ is a pretrained reverse-model evaluation, $q$ is an amortized proposal, $r$ is a reward-model evaluation, $\nabla v$ is a value-gradient evaluation, $N$ is the number of particles, $T$ is the diffusion horizon, $T'$ is the guidance budget, and $J$ is the number of rollout samples for value estimation (see Algorithm \ref{alg:sparsesmc} for details).}
\vspace{0.1cm}
\label{tab:cost}
\setlength{\tabcolsep}{8pt}
\renewcommand{\arraystretch}{1.15}
\begin{tabular}{l c c}
\toprule
\textbf{Variant}
  & \textbf{Full Guidance}
  & \textbf{Sparse Guidance} \\
\midrule

\textsc{SMC-base}
  & $NT \times p,\ NTJ \times r$
  & $NT \times p,\ NT'J \times r$ \\

\textsc{SMC-grad}
  & $NT \times p,\ NTJ \times r,\ NT \times \nabla v$
  & $NT \times p,\ NT'J \times r,\ NT' \times \nabla v$ \\


\bottomrule
\end{tabular}
\end{table}

\section{Experiments}

To evaluate the effectiveness of \method, we focus our main experiments on
Masked Diffusion Models (MDMs), with experiments on other diffusion models
provided in Appendix~\ref{Expdetail}.
We compare \textbf{\method} against full guidance and several sparse scheduling baselines.
Specifically, \textbf{Interval-$k$} ($k=1,\ldots,5$) allocates all guidance
steps to the $k$-th equally sized segment of the reverse diffusion trajectory,
where Interval-1 corresponds to the earliest denoising stage and Interval-5
to the latest.
\textbf{Uniform} distributes the guidance steps evenly across the entire trajectory.
\textbf{Top-$V$} selects the timesteps with the largest estimated values
$\hat{V}_t$, while \textbf{Top-$\Delta V$} selects the timesteps with the
largest consecutive value differences $\hat{V}_t-\hat{V}_{t+1}$.
Both value-based baselines use the same warmup value statistics as \method.
Unless otherwise stated, all reported times correspond to the per-generation
cost after schedule selection; warmup costs are reported separately.
Comprehensive experimental details are available in Appendix~\ref{Expdetail}.

\subsection{Synthetic Experiment: Binary MNIST}\label{sec:binarymnist}

We first evaluate \method on Binary MNIST using the MDM with the SMC-grad
proposal, based on the codebase of \cite{ou2025inference}.
The reward is defined as $r(x_0)=\log p(c\mid x_0)$, where $c$ denotes the
condition ``the digit is 4''.
We use a total horizon of $T=100$ with $\alpha=1.0$, $N=20$ particles,
$J=1$ rollout for value estimation, and $M=3$ full-guidance warmup generations.
A sample $x_0$ is considered ``valid'' if $r(x_0)>-0.1$.
To measure diversity, \textit{Unique Valid} counts the number of distinct
valid samples, where samples with Jaccard similarity $\geq 0.85$ are treated
as identical.
Results are averaged over 50 independent seeds, yielding 1,000 generated
samples per schedule.
The warmup phase took 24.7\,s.
For evaluation, $\mathbf{UTD}$ was computed with $R_{\max}=1$, which is valid
since $R(x_0)=p(c\mid x_0)\leq 1$.
Figure~\ref{fig:mnistgraph} summarizes the results, and complete numerical
results, including the SMC-base setting, are provided in
Appendix~\ref{binarymnistappendix}.

\begin{figure}[htbp]
  \centering
  \begin{minipage}[t]{0.32\textwidth}
    \centering
    \vspace{0pt} 
    \includegraphics[trim={0 0 0 0}, clip, width=\textwidth]{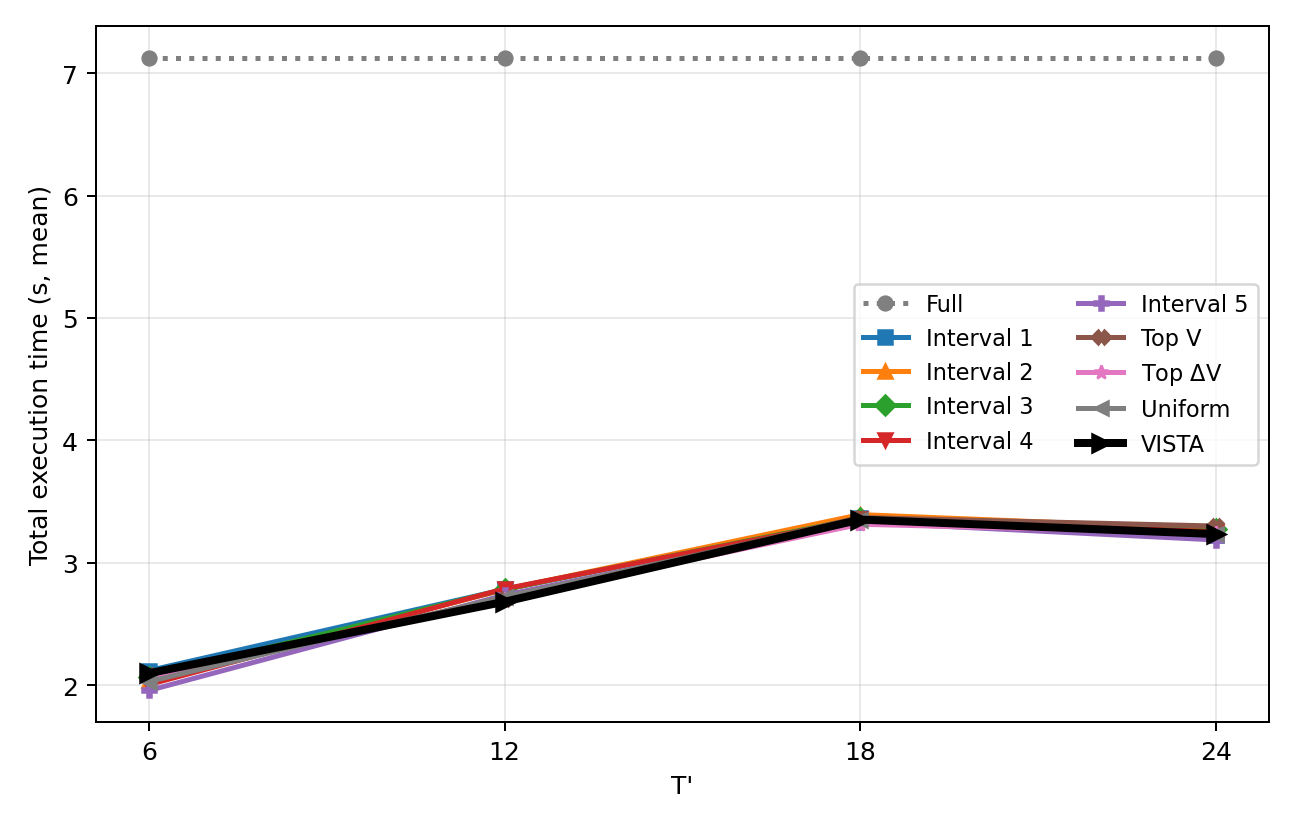}
    \label{fig:mnisttime}
  \end{minipage}
  \hfill 
  \begin{minipage}[t]{0.32\textwidth}
    \centering
    \vspace{0pt} 
    \includegraphics[trim={0 0 0 0}, clip, width=\textwidth]{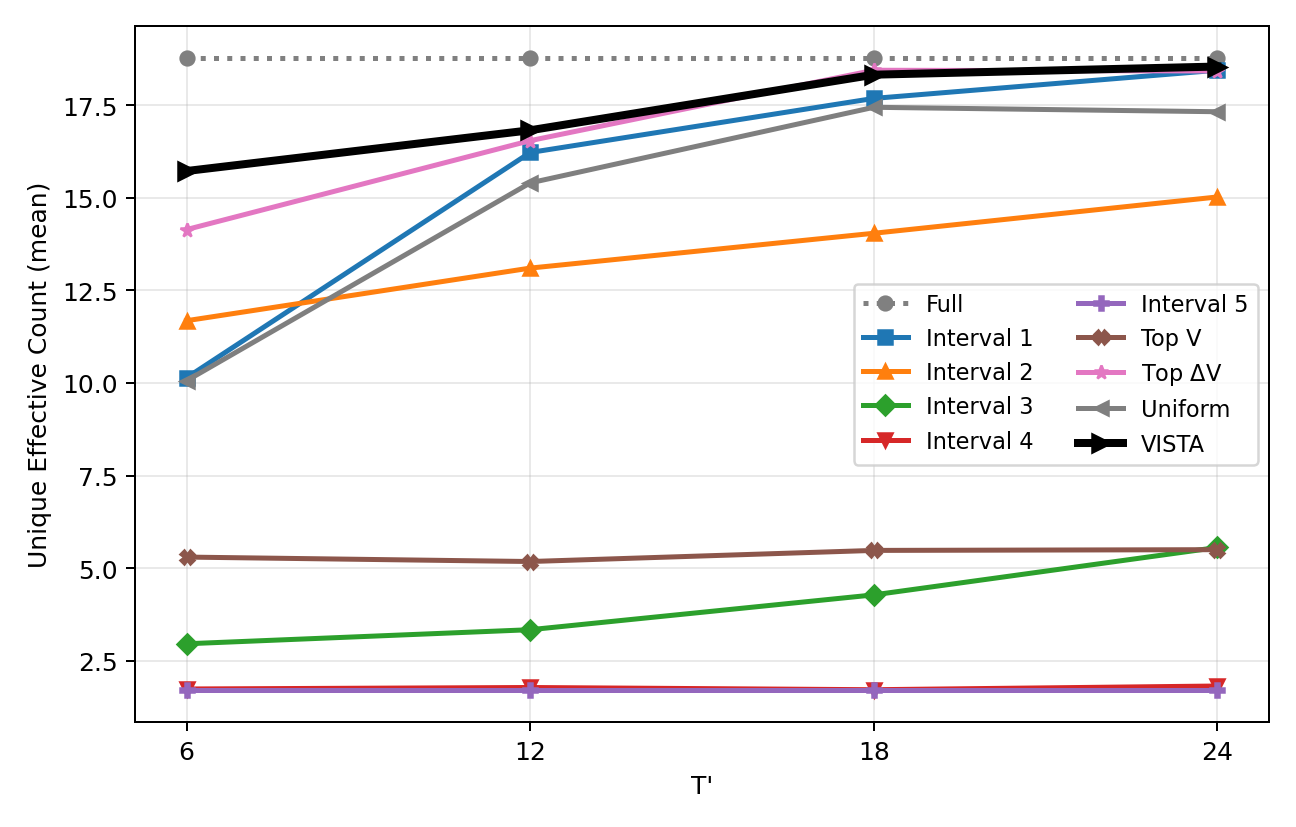}
    \label{fig:mnistuval}
  \end{minipage}
  \hfill 
  \begin{minipage}[t]{0.32\textwidth}
    \centering
    \vspace{0pt} 
    \includegraphics[trim={0 0 0 0}, clip, width=\textwidth]{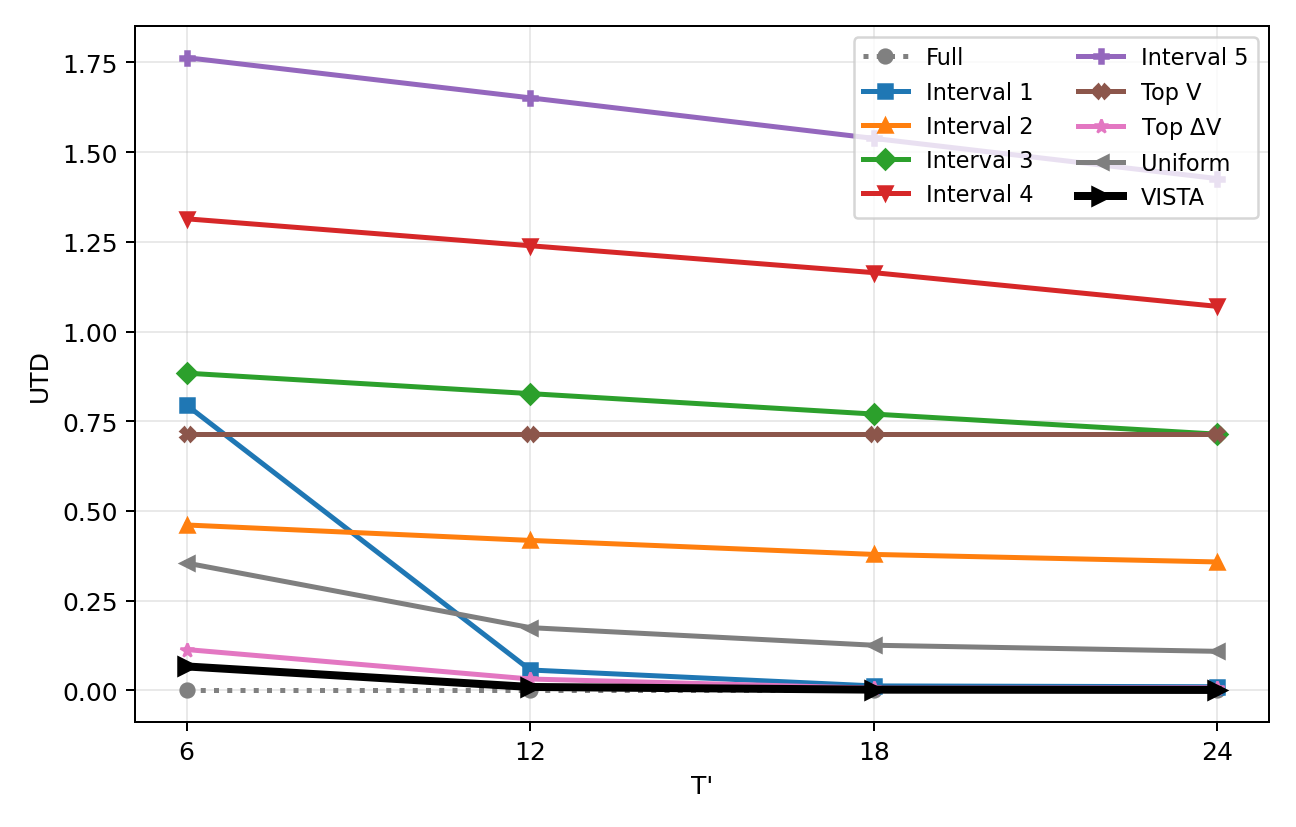}
    \label{fig:mnistutd}
  \end{minipage}
  \caption{Binary MNIST (SMC-grad). Execution time, Unique Valid, and UTD as the guidance budget $T'$ varies.}
  \label{fig:mnistgraph}
\end{figure}

\subsection{Toxic Text Generation}\label{sec:toxictext}

Toxic text generation is a high-dimensional discrete task where the reward
model must guide the sampler against a strong language prior.
The reward is defined as $r(x_0)=\log p(c\mid x_0)$ for a target toxicity
condition.
Without guidance, the base model satisfies the target toxicity condition in
only $0.8\%$ of generations.
We report \textit{Unique Toxic Ratio}, defined as the fraction of generated
samples that are both toxic and unique, where samples with normalized
Levenshtein distance $\leq 0.05$ are treated as identical.
We set $T=100$, $\alpha=0.2$, $N=8$, $J=4$, and $M=3$, and use the
SMC-grad proposal.
For each of 15 prompts, we perform two independent runs, yielding
$15\times2\times8=240$ generated samples for each scheduling method.
The warmup phase required 639.6\,s.
Figure~\ref{fig:toxicgraph} summarizes the results, and complete numerical
results are provided in Appendix~\ref{toxictextappendix}.
Our implementation is based on the codebase of \cite{ou2025inference}.

\begin{figure}[htbp]
  \centering
  \begin{minipage}[t]{0.32\textwidth}
    \centering
    \vspace{0pt}
    \includegraphics[trim={0 0 0 0}, clip, width=\textwidth]{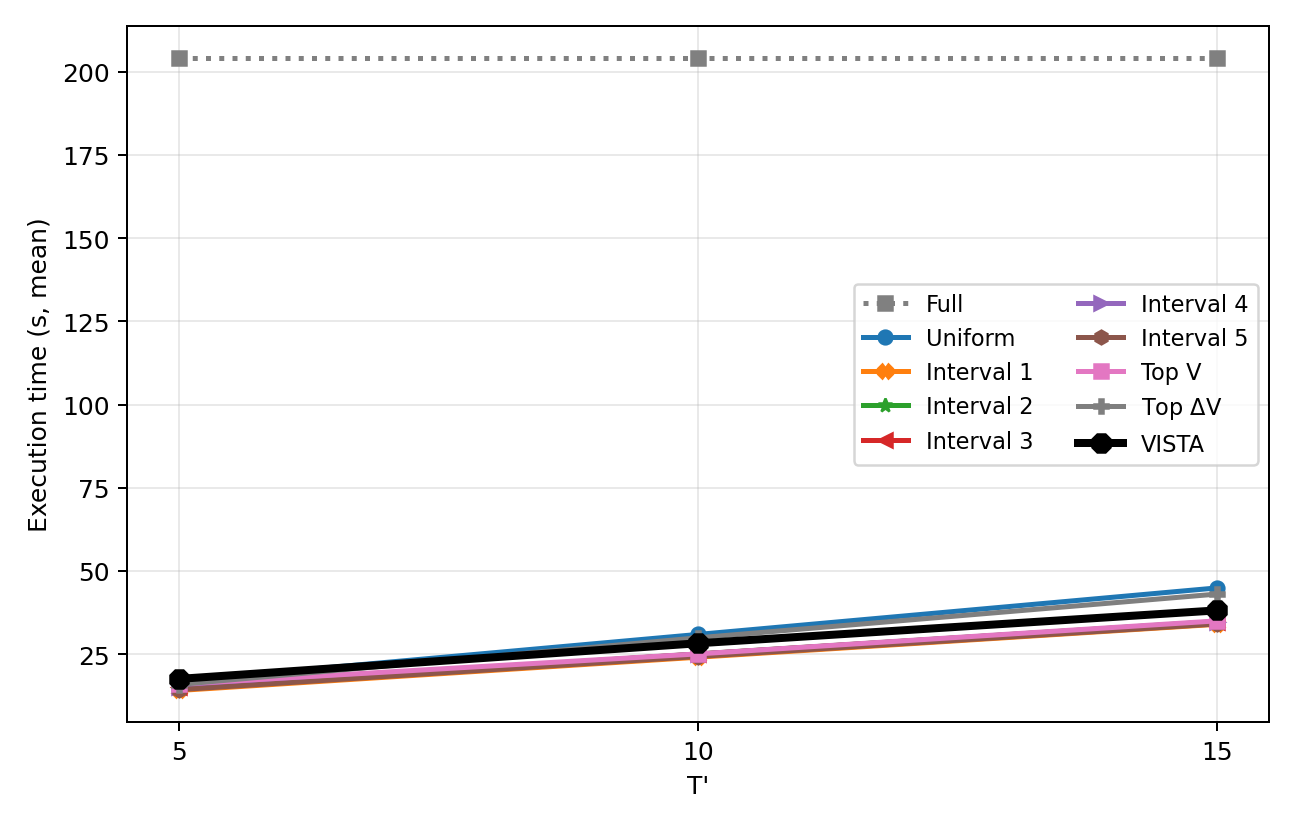}
  \end{minipage}
  \hfill
  \begin{minipage}[t]{0.32\textwidth}
    \centering
    \vspace{0pt}
    \includegraphics[trim={0 0 0 0}, clip, width=\textwidth]{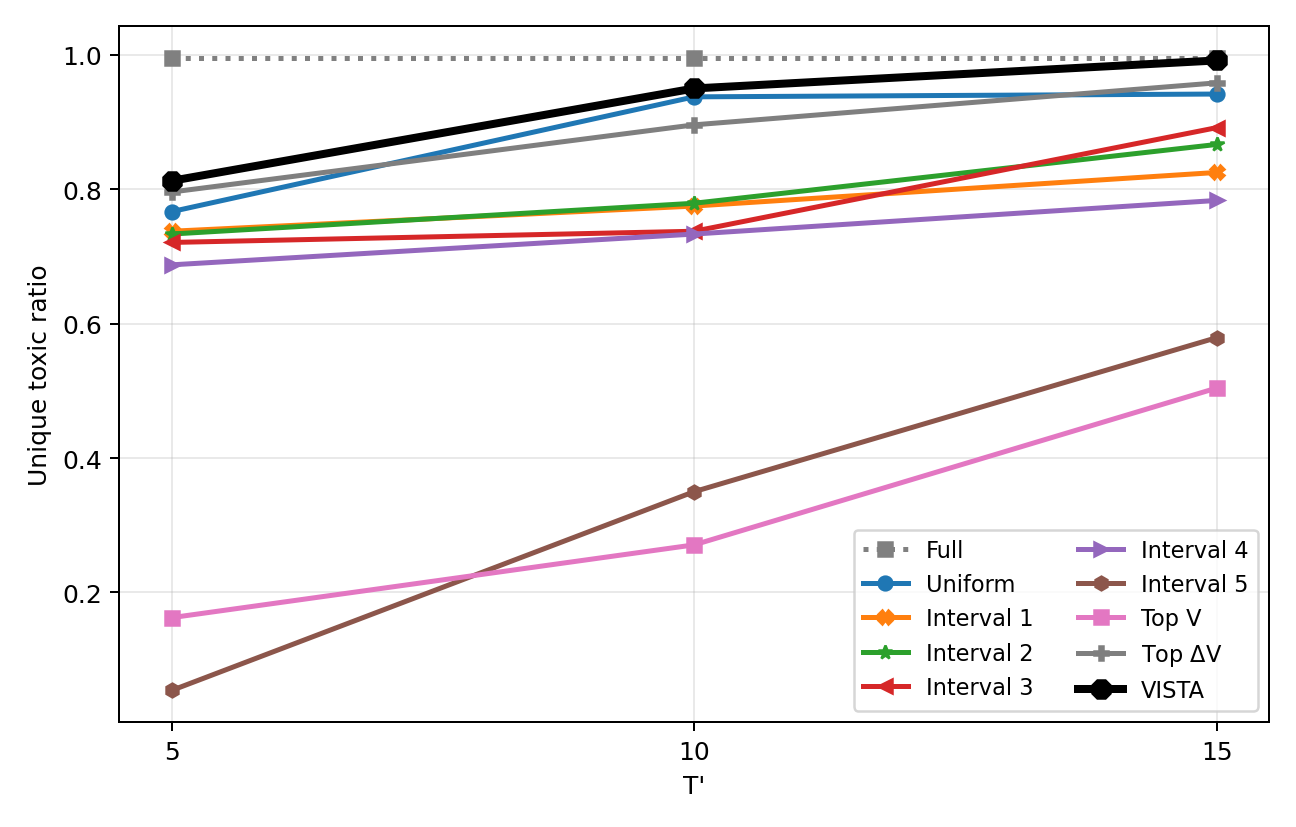}
  \end{minipage}
  \hfill
  \begin{minipage}[t]{0.32\textwidth}
    \centering
    \vspace{0pt}
    \includegraphics[trim={0 0 0 0}, clip, width=\textwidth]{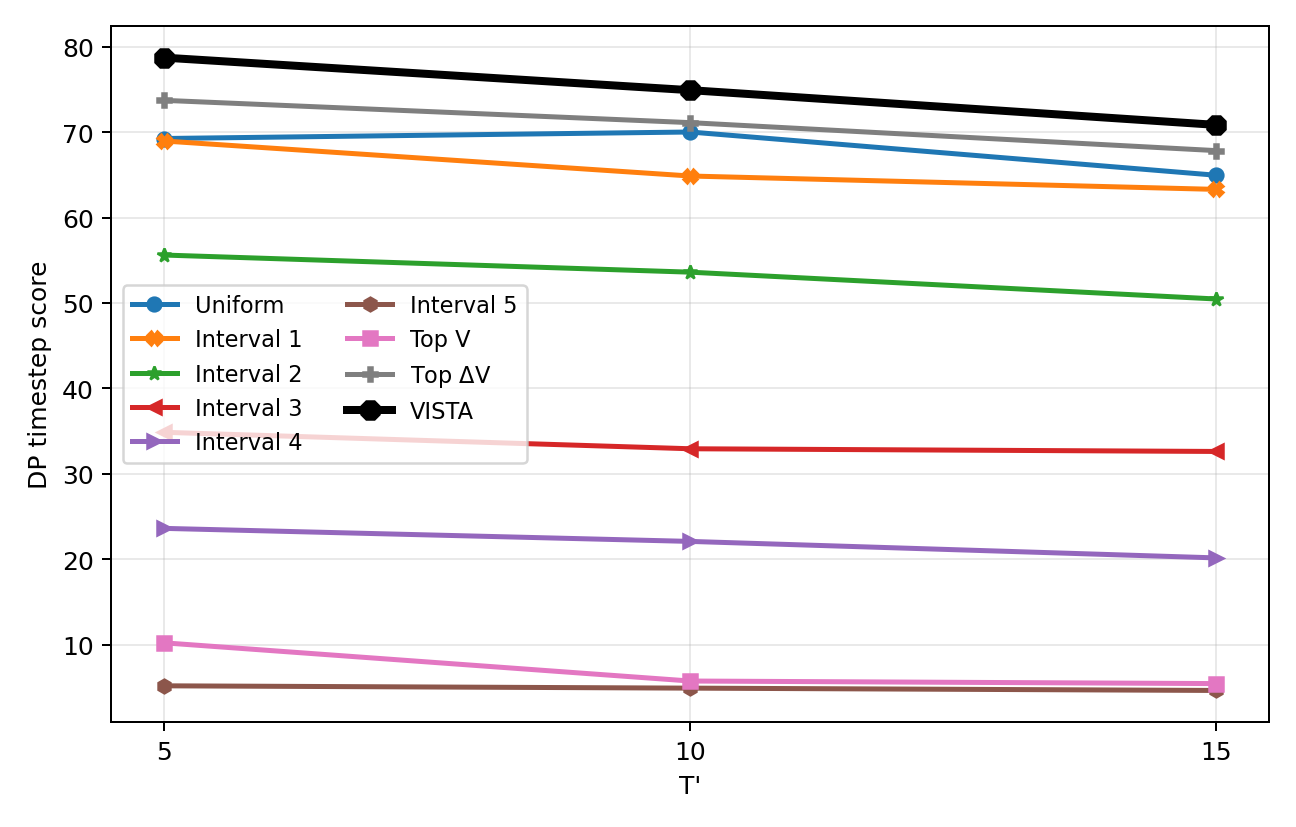}
  \end{minipage}
  \caption{Toxic Text (SMC-grad).
  Execution time, Unique Toxic Ratio, and $\hat{S}(\mathcal{G})$ as the
  guidance budget $T'$ varies.}
  \label{fig:toxicgraph}
\end{figure}

\subsection{CIFAR-10 Image Generation}

We evaluate CIFAR-10 generation experiment using the MDM, SMC-base framework, based on the codebase of \cite{schiff2024simple}. The reward is defined as $r(x_0) = \log p(c|x_0)$, where $c$ represents the condition ``the image is an airplane''. We use $T=1,000$ time steps, $N=20$ particles, $\alpha=1.0$, $J=5$ rollout for value estimation, and $M=5$ warmup simulations. A sample $x_0$ is considered ``valid'' if $p(c|x_0) > 0.5$. \textit{Unique Valid} is number of valid samples, where samples with LPIPS distance $\leq 0.03$ are considered identical. The warmup phase took 217.9 s. Results are averaged over 50 runs, totaling $1,000$ generated samples per schedule. The graph of experiment result is in Figure \ref{fig:cifargraph}.

\begin{figure}[htbp]
  \centering
  \begin{minipage}[t]{0.32\textwidth}
    \centering
    \vspace{0pt} 
    \includegraphics[trim={0 0 0 0}, clip, width=\textwidth]{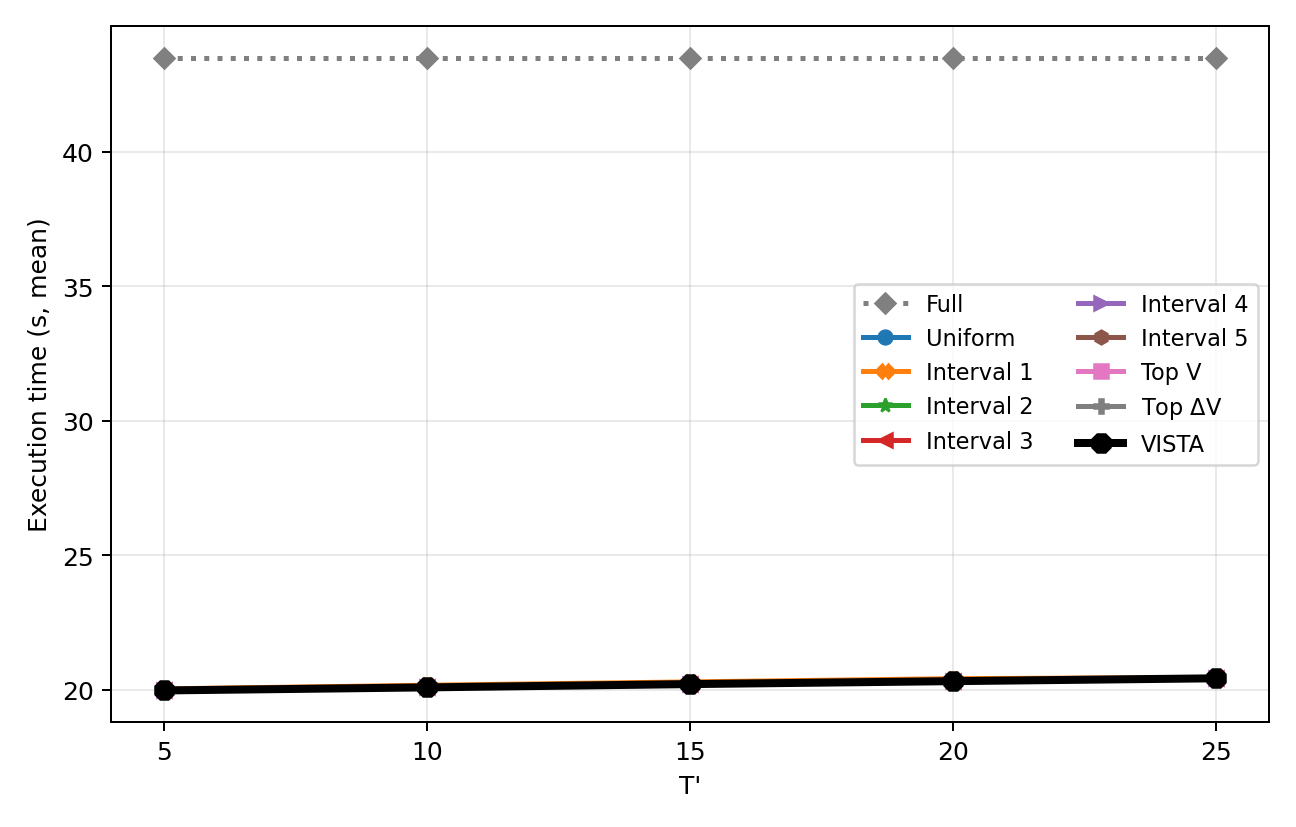}
    \label{fig:cifartime}
  \end{minipage}
  \hfill 
  \begin{minipage}[t]{0.32\textwidth}
    \centering
    \vspace{0pt} 
    \includegraphics[trim={0 0 0 0}, clip, width=\textwidth]{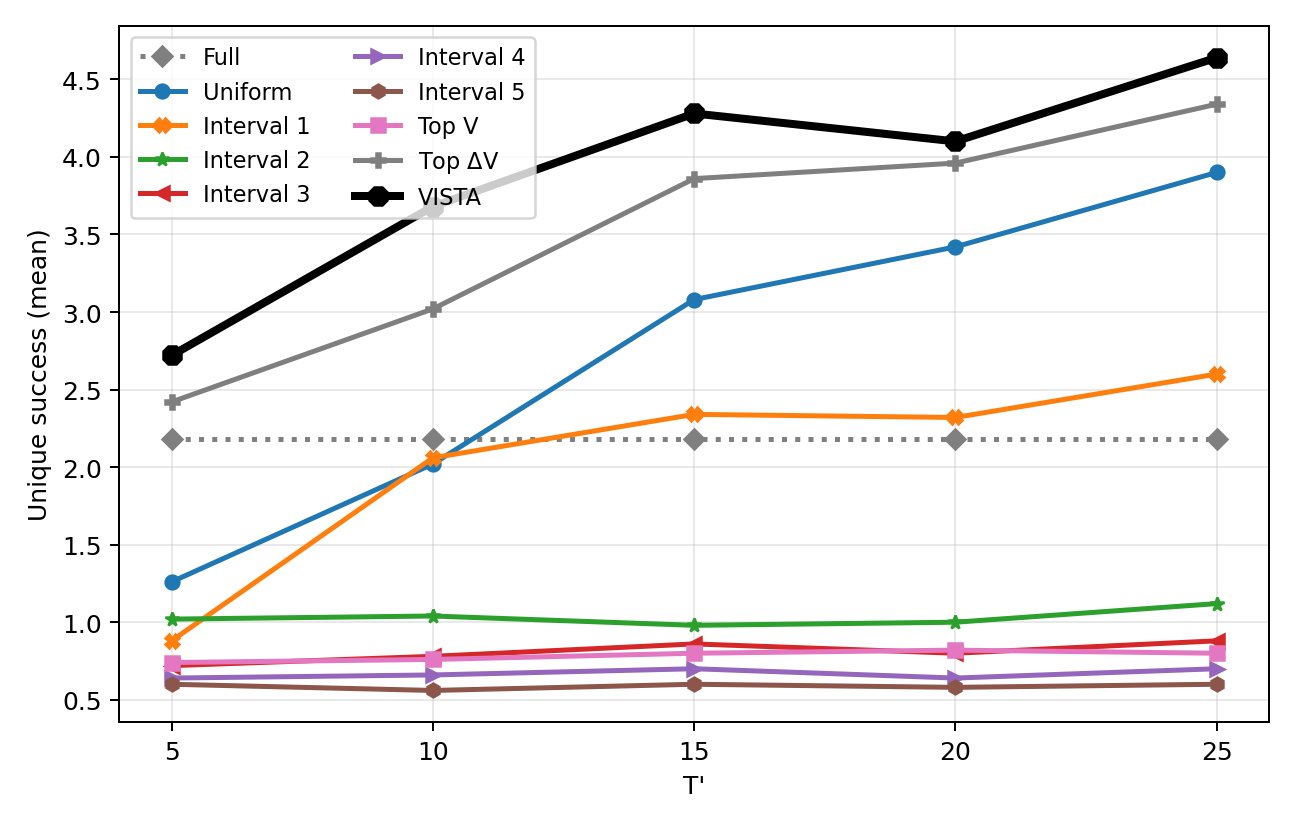}
    \label{fig:cifaruval}
  \end{minipage}
  \hfill 
  \begin{minipage}[t]{0.32\textwidth}
    \centering
    \vspace{0pt} 
    \includegraphics[trim={0 0 0 0}, clip, width=\textwidth]{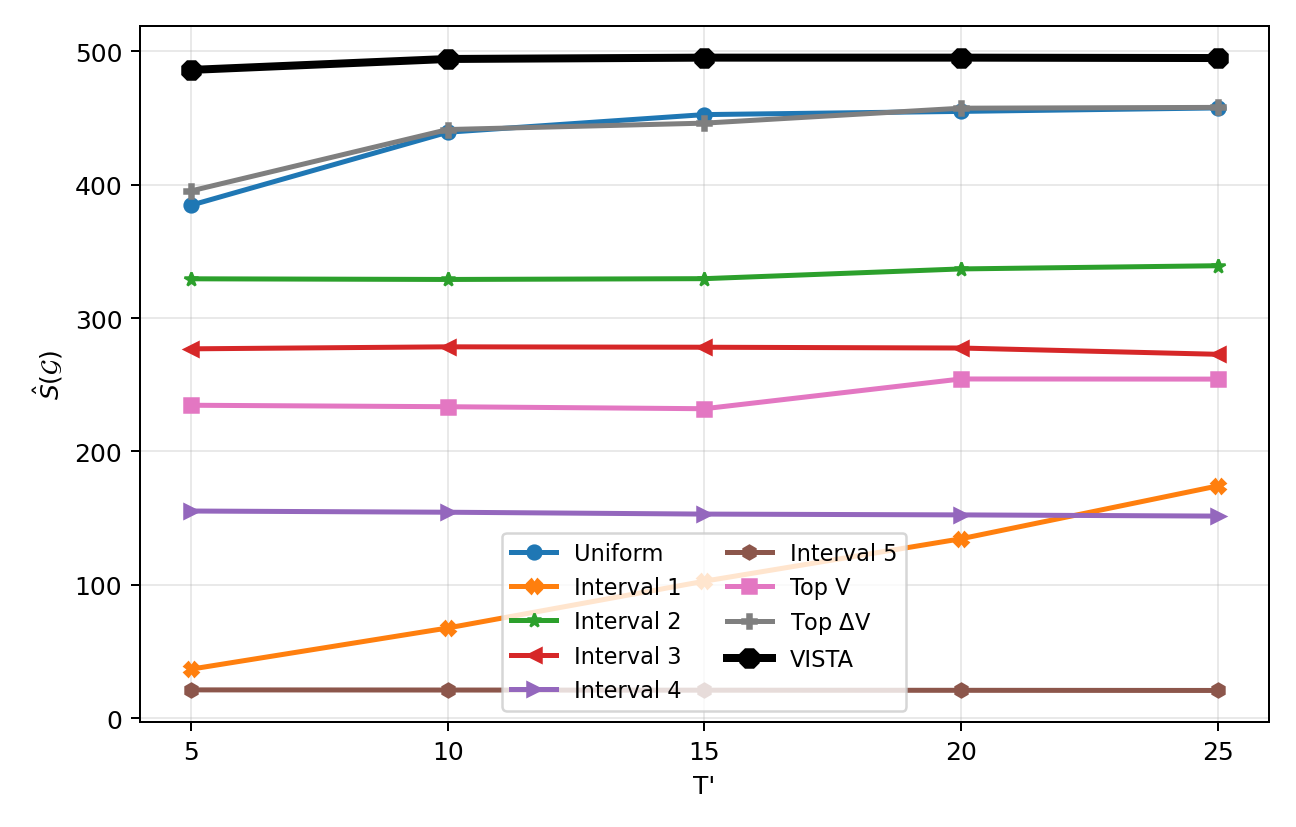}
    \label{fig:cifarscore}
  \end{minipage}
  \caption{CIFAR-10 (SMC-base). Execution time, unique valid samples, and $\hat{S}(\mathcal{G})$ as the guidance budget $T'$ varies.}
  \label{fig:cifargraph}
\end{figure}

\subsection{Biological Sequence Design}\label{biologydesign}

For DNA generation, we follow the experimental setup of
\citep{ou2025inference}.
We use the reward $r(x_0)=\log p(c\mid x_0)$, where $p(c\mid x_0)$ denotes
the predicted ATAC accessibility of the generated sequence.
Following \citep{zekri2026fine}, a sequence is considered biologically active
if its ATAC-Acc score exceeds $0.5$.
We define \textit{Unique Valid} as the number of distinct biologically active
sequences, where two sequences with Hamming distance $\leq 1$ are treated as
identical.
We set $T=128$, $\alpha=0.01$, $N=8$, $J=4$, and $M=3$.
Results are averaged over 30 independent runs, yielding 240 generated
sequences for each scheduling method.
The warmup phase took 49.1\,s.

For Protein Fitness, we use the pretrained masked diffusion model from
\citep{sahoo2024simple} together with a fitness oracle as the reward model
$r(x_0)$. A sequence is considered valid if $r(x_0)\geq 0.5$, and
\textit{Unique Valid} counts the number of distinct valid sequences, where
two sequences with Hamming distance $\leq 1$ are treated as identical.
We set $T=128$, $\alpha=0.3$, $N=32$, $J=3$, and $M=3$.
Results are averaged over 20 independent runs, yielding 640 generated
sequences for each scheduling method. the warmup phase took 19.5\,s.

The main results are presented in
Tables~\ref{tab:dna-time-atac-score} and~\ref{tab:protein-mdlm-smc-base},
with the corresponding schedules shown in
Appendix~\ref{biologyappendix}.

\begin{table*}[t]
\centering
\begin{minipage}[t]{0.485\textwidth}
\centering
\small
\caption{DNA sequence generation (SMC-grad): execution time, unique valid, and $\hat{S}(\mathcal{G})$.}
\label{tab:dna-time-atac-score}
\resizebox{\linewidth}{!}{
\setlength{\tabcolsep}{4.0pt}
\renewcommand{\arraystretch}{0.90}
\begin{tabular}{llccc}
\toprule
\textbf{$T'$} & \textbf{Policy} & {\textbf{Time [s]} $\downarrow$} & {\textbf{Unique Valid} $\uparrow$} & {$\hat{S}(\mathcal{G})$ $\uparrow$} \\
\midrule
\multirow{1}{*}{128}  & Full  & $15.34 \pm 0.97$  & $\mathbf{5.90 \pm 1.35}$  & {--} \\
\midrule
\multirow{9}{*}{8}  & Interval 1  & $2.61 \pm 0.63$  & $0.23 \pm 0.42$  & $1.89\times10^{-4}$ \\
  & Interval 2  & $2.54 \pm 0.54$  & $0.20 \pm 0.48$  & $2.89\times10^{-6}$ \\
  & Interval 3  & $2.69 \pm 0.66$  & $0.60 \pm 0.76$  & $3.77$ \\
  & Interval 4  & $2.87 \pm 0.76$  & $1.03 \pm 1.62$  & $4.57$ \\
  & Interval 5  & $2.78 \pm 0.75$  & $1.27 \pm 1.15$  & $1.43\times10^{-6}$ \\
  & Top-$V$  & $2.66 \pm 0.73$  & $1.27 \pm 1.15$  & $1.43\times10^{-6}$ \\
  & Top-$\Delta V$  & $2.94 \pm 0.74$  & $2.30 \pm 1.99$  & $8.95$ \\
  & Uniform  & $2.79 \pm 0.66$  & $0.83 \pm 0.52$  & $7.77$ \\
  & VISTA  & $2.84 \pm 0.61$  & $\mathbf{2.73 \pm 1.90}$  & $\mathbf{12.40}$ \\
\midrule
\multirow{9}{*}{24}  & Interval 1  & $5.04 \pm 0.87$  & $0.40 \pm 0.61$  & $2.64\times10^{-5}$ \\
  & Interval 2  & $5.09 \pm 0.70$  & $0.77 \pm 1.02$  & $0.15$ \\
  & Interval 3  & $5.09 \pm 0.72$  & $1.50 \pm 1.26$  & $3.37$ \\
  & Interval 4  & $4.96 \pm 0.68$  & $3.37 \pm 1.85$  & $2.66$ \\
  & Interval 5  & $4.93 \pm 0.84$  & $3.20 \pm 1.99$  & $1.24\times10^{-6}$ \\
  & Top-$V$  & $5.05 \pm 0.70$  & $3.07 \pm 1.41$  & $0.19$ \\
  & Top-$\Delta V$  & $5.63 \pm 0.80$  & $\mathbf{4.93 \pm 1.84}$  & $9.58$ \\
  & Uniform  & $5.51 \pm 0.80$  & $2.70 \pm 1.39$  & $7.10$ \\
  & VISTA  & $5.29 \pm 0.81$  & $4.17 \pm 1.21$  & $\mathbf{12.08}$ \\
\bottomrule
\end{tabular}
}
\end{minipage}
\hfill
\begin{minipage}[t]{0.485\textwidth}
\centering
\small
\caption{Protein Fitness generation (SMC-base): execution time, unique valid, and $\hat{S}(\mathcal{G})$.}
\label{tab:protein-mdlm-smc-base}
\resizebox{\linewidth}{!}{
\setlength{\tabcolsep}{4.0pt}
\renewcommand{\arraystretch}{0.70}
\begin{tabular}{llccc}
\toprule
\textbf{$T'$} & \textbf{Policy} & {\textbf{Time [s]} $\downarrow$} & {\textbf{Unique Valid} $\uparrow$} & {$\hat{S}(\mathcal{G})$ $\uparrow$} \\
\midrule
\multirow{1}{*}{128}  & Full  & $3.12$  & $\mathbf{9.00 \pm 2.77}$  & {--} \\
\midrule
\multirow{9}{*}{16}  & Interval 1  & $1.531 \pm 0.003$  & $13.20 \pm 3.24$  & $3.79\times10^{3}$ \\
  & Interval 2  & $1.530 \pm 0.004$  & $10.65 \pm 4.44$  & $2.95\times10^{3}$ \\
  & Interval 3  & $1.530 \pm 0.002$  & $8.10 \pm 2.15$  & $2.07\times10^{3}$ \\
  & Interval 4  & $1.529 \pm 0.002$  & $6.00 \pm 2.20$  & $1.07\times10^{3}$ \\
  & Interval 5  & $1.517 \pm 0.003$  & $5.35 \pm 2.16$  & $1.12 \times 10^{2}$ \\
  & Top-$V$  & $1.519 \pm 0.002$  & $8.40 \pm 2.68$  & $3.34\times10^{3}$ \\
  & Top-$\Delta V$  &  $1.531 \pm 0.002$  & $12.60 \pm 2.58$  & $4.23\times10^{3}$ \\
  & Uniform  & $1.521 \pm 0.003$  & $10.85 \pm 3.79$  & $3.83\times10^{3}$ \\
  & VISTA  & $1.530 \pm 0.004$  & $\mathbf{14.15 \pm 3.80}$  & $\mathbf{4.53\times10^{3}}$ \\
\midrule
\multirow{9}{*}{32}  & Interval 1  & $1.777 \pm 0.003$  & $\mathbf{12.60 \pm 4.04}$  & $3.57\times10^{3}$ \\
  & Interval 2  & $1.777 \pm 0.005$  & $9.60 \pm 4.58$  & $2.81\times10^{3}$ \\
  & Interval 3  & $1.776 \pm 0.004$  & $6.80 \pm 3.24$  & $1.82\times10^{3}$ \\
  & Interval 4  & $1.773 \pm 0.005$  & $6.00 \pm 2.10$  & $9.15\times 10^{2}$ \\
  & Interval 5  & $1.760 \pm 0.006$  & $6.10 \pm 2.47$  & $96.0$ \\
  & Top-$V$  & $1.762 \pm 0.003$  & $8.55 \pm 2.95$  & $2.62\times10^{3}$ \\
  & Top-$\Delta V$  & $1.774 \pm 0.008$  & $11.25 \pm 3.49$  & $3.56\times10^{3}$ \\
  & Uniform  & $1.768 \pm 0.004$  & $11.80 \pm 2.82$  & $3.33\times10^{3}$ \\
  & VISTA  & $1.776 \pm 0.002$  & $12.00 \pm 3.32$  & $\mathbf{3.94\times10^{3}}$ \\
\bottomrule 
\end{tabular}
}
\end{minipage}
\end{table*}

\subsection{Discussion}

The schedules selected by \method differ substantially across tasks,
highlighting the importance of adapting guidance timing to the underlying
model and reward function.
Binary MNIST places most guidance steps in the early part of the reverse
process, whereas Toxic Text, DNA, and CIFAR-10 distribute guidance more broadly
across the trajectory.
These distinct patterns suggest that a single fixed temporal schedule is
unlikely to be uniformly effective across tasks.

Interestingly, sparse guidance can sometimes outperform full guidance on
diversity-sensitive metrics, as observed on CIFAR-10.
This is consistent with the possibility that overly frequent correction can
reduce particle diversity in finite-particle SMC.
Overall, the results indicate that selectively allocating guidance steps can
provide a favorable trade-off between downstream performance and inference
cost.

Finally, the weighted extension in Appendix~\ref{Expdetail} shows that the
same schedule-selection principle can also be applied beyond masked diffusion
through non-uniform temporal weighting.

\vspace{-0.3cm}
\section{Limitation}

\method has three main limitations.
First, \method requires a small number of full-guidance warmup generations
before schedule optimization.
These generations can be retained as model outputs while simultaneously
providing the value statistics used for schedule selection, but they still
introduce an upfront computational cost.
Therefore, the benefit of \method is most pronounced when the selected
schedule is reused across multiple subsequent generations.

Second, the UTD objective is a conservative upper bound on path deviation
rather than a direct objective for downstream task performance.
As a result, the schedule that optimizes $\hat{S}(\mathcal{G})$ does not
necessarily maximize the downstream evaluation metric at every guidance
budget.
Tighter surrogates that better capture finite-particle effects and downstream
sample quality are an important direction for future work.

Third, the weighted extension for non-masked diffusion requires a
user-specified temporal weighting function $\lambda(t)$.
Our experiments show that non-uniform weighting can extend \method beyond
masked diffusion, but the appropriate choice of $\lambda(t)$ may depend on
the diffusion process and task.
Automatically learning or adapting these temporal weights remains an open
problem.

\vspace{-0.3cm}
\section{Conclusion}
We introduced \method, a training-free scheduler for sparse reward-guided SMC in discrete diffusion. By minimizing an estimable upper bound on path deviation, \method identifies task-specific guidance schedules under a fixed computational budget. Our experiments across image, text, and biological domains demonstrate that automatic schedule selection can retain strong downstream performance, and in several settings even outperform full guidance, while substantially reducing inference latency. Ultimately, this work shows that for efficient discrete diffusion alignment, choosing \emph{when} to guide is as important as designing \emph{how} to guide.

\section*{Acknowledgments}

This work was partly supported by the KHIDI grant funded by the Korean government (MOHW) [No.RS-2025-02307233, No.RS-2026-25613012], the NRF or IITP grants funded by the Korean government (MSIT) [No.05-26-04-0094, No.RS-2026-25472075, No.RS-2025-02305581, No.RS-2025-25442338, and No.RS-2021-II211343], the ITIP grant funded by the Korean government (MOTIR) [No.RS-2026-25549946], the Research grant from SNU, and the Strategic Hub grant for International Research Collaboration of SNU. Kyungsu Kim is affiliated with the School of Transdisciplinary Innovations, Department of Biomedical Science, Interdisciplinary Program in AI (IPAI) \& Graduate School of AI, Medical Research Center, and AI Institute at SNU.

\bibliography{main}
\bibliographystyle{plainnat}

\clearpage
\appendix

{\huge \centering \bfseries Appendix \par}
\addcontentsline{toc}{section}{Appendix}

\startcontents[appendix]
\printcontents[appendix]{}{1}{}

\clearpage

\section{LLM Usage Disclosure}
\label{app:llm_usage}

We used a large-language model (LLM) only for minor text editing, such as correcting typos and adjusting wording to a formal academic tone. The LLM was not involved in research ideation, experimental design, implementation, or analysis. All scientific content and claims were conceived and written by the authors, who take full responsibility for the entire paper.

\section{Theoretical Details}

\subsection{Background}\label{appendixbackground}

\begin{lemma}\label{bellman}
    \begin{equation}
        V_t(x_t) = \mathbb{E}_{x_{s}\sim p_{s | t}(\cdot | x_t)}[V_{s}(x_{s})]
    \end{equation}
    where $0\leq s <t\leq T$.
    Also,
    \begin{equation}
        \mathbb{E}_{x_t \sim p_t}[V_t(x_t)] = \mathbb{E}_{x_0 \sim p_0}[R(x_0)].
    \end{equation}
\end{lemma}

\begin{proof}
    \begin{equation}
        V_t(x_t) = \mathbb{E}_{x_0 \sim p_{0|t}(\cdot | x_t)}[R(x_0)]
    \end{equation}
    comes from the definition of value function.
    
    \begin{align}
        &\mathbb{E}_{x_{s}\sim p_{s | t}(\cdot | x_t)}[V_{s}(x_{s})]
        = \mathbb{E}_{x_{s}\sim p_{s | t}(\cdot | x_t)}[ \mathbb{E}_{x_0 \sim p_{0|s}(\cdot | x_{s})}[R(x_0)]] \\
        =& \mathbb{E}_{x_0 \sim p_{0|t}(\cdot | x_t)}[R(x_0)] = V_t(x_t).
    \end{align}

    \begin{align}
        \mathbb{E}_{x_T\sim p_T}[V_T(x_T)] = \mathbb{E}_{x_T\sim p_T}[\mathbb{E}_{x_0 \sim p_{0|T}(\cdot | x_T)}[R(x_0)]]=\mathbb{E}_{x_0 \sim p_0}[R(x_0)].
    \end{align}
\end{proof}

\begin{theorem}
\label{optimalpolicy}
    $\{\pi_t\}_{t=T}^0$ that maximizes
    \begin{equation}
        \mathbb{E}_{\{\pi_t\}_{t=T}^0} [r(x_0) - \alpha \sum_{t=T-1}^0 \KL(\pi_{t | t+1}(\cdot | x_{t+1}) || p_{t | t+1}(\cdot | x_{t+1})) - \alpha\KL(\pi_T(\cdot) || p_T(\cdot))]
    \end{equation}
    ($\alpha>0$) is
    \begin{equation}
        \pi_t(x_t) = \frac{1}{Z}V_t(x_t)p_t(x_t). \quad (Z:=\mathbb{E}_{x_0\sim p_0}[R(x_0)])
    \end{equation}
\end{theorem}

\begin{proof}
    We define the Markov process $\{\pi_t^*\}_{t=T}^0$ as
    \begin{align}
        \pi_T^*(x_T) &= \frac{1}{Z}V_T(x_T)p_T(x_T), \\
        \pi_{t | t+1}^*(x_t | x_{t+1}) &= \frac{V_t(x_t)p_{t|t+1}(x_t | x_{t+1})}{V_{t+1}(x_{t+1})} \quad (t=0, 1, \cdots, T-1).
    \end{align}
    (These distributions are well-defined(normalized) by Lemma \ref{bellman}.)

    We observe
    \begin{align}
        &\mathbb{E}_{\{\pi_t\}_{t=T}^0} [r(x_0) - \alpha \sum_{t=T-1}^0 \KL(\pi_{t | t+1}(\cdot | x_{t+1}) || p_{t | t+1}(\cdot | x_{t+1})) - \alpha\KL(\pi_T(\cdot) || p_T(\cdot))] \\
        =& \mathbb{E}_{x_0\sim \pi_0}[r(x_0)] -\alpha\sum_{t=T-1}^0 \mathbb{E}_{x_{t}, x_{t+1}\sim \pi_{t, t+1}} [\log\frac{\pi_{t|t+1}(x_t|x_{t+1})}{p_{t|t+1}(x_t|x_{t+1})}] -\alpha\mathbb{E}_{x_T\sim \pi_T}[\log\frac{\pi_T(x_T)}{p_T(x_T)}] \\
        =& \mathbb{E}_{x_T\sim \pi_T} \mathbb{E}_{x_{T-1}\sim \pi_{T-1|T}(\cdot|x_T)} \cdots \mathbb{E}_{x_0\sim \pi_{0|1}(\cdot | x_1)} \notag \\
        & \quad [r(x_0) - \alpha \sum_{t=T-1}^0 \log \frac{\pi_{t|t+1}(x_t | x_{t+1})}{p_{t|t+1}(x_t | x_{t+1})} - \alpha\log\frac{\pi_T(x_T)}{p_T(x_T)}]. \label{recurobj}
    \end{align}

    If we use the recurrence relation
    \begin{align}
        &\mathbb{E}_{x_s \sim \pi_{s | s+1}(\cdot | x_{s+1})}[v_s(x_s) - \alpha\log\frac{\pi_{s|s+1}(x_s|x_{s+1})}{p_{s|s+1}(x_s|x_{s+1})}] \\
        =& \mathbb{E}_{x_s \sim \pi_{s | s+1}(\cdot | x_{s+1})}[v_{s+1}(x_{s+1}) - \alpha\log\frac{\pi_{s|s+1}(x_s|x_{s+1})}{\pi_{s|s+1}^*(x_s|x_{s+1})}] \\
        =& v_{s+1}(x_{s+1}) - \alpha \KL(\pi_{s|s+1}(\cdot|x_{s+1}) ||\pi_{s|s+1}^*(\cdot|x_{s+1}))
    \end{align}
    consequently, Eq. \eqref{recurobj} could be simplified into
    \begin{align}
        &\mathbb{E}_{x_T\sim \pi_T}[v_T(x_T)-\alpha\log\frac{\pi_T(x_T)}{p_T(x_T)}] -\alpha\sum_{t=T-1}^0 \mathbb{E}_{x_{t+1} \sim \pi_{t+1}}[\KL(\pi_{t|t+1}(\cdot | x_{t+1}) || \pi_{t|t+1}^*(\cdot | x_{t+1}))] \\
        =& \alpha\log Z - \alpha\sum_{t=T-1}^0 \mathbb{E}_{x_{t+1} \sim \pi_{t+1}}[\KL(\pi_{t|t+1}(\cdot | x_{t+1}) || \pi_{t|t+1}^*(\cdot | x_{t+1}))] -\alpha\KL(\pi_T(\cdot) || \pi_T^*(\cdot)).
    \end{align}

    We now easily see that this objective is maximized iff $\pi_T = \pi_T^*$ and $\pi_{t|t+1} = \pi_{t|t+1}^*$ for all $t=0, 1, \cdots, T-1$.

    The final step is proving
    \begin{equation}
        \pi_t^*(x_t) = \frac{1}{Z}V_t(x_t)p_t(x_t).
    \end{equation}
    This is satisfied at $t=T$. If
    \begin{equation}
        \pi_{t+1}^*(x_{t+1}) = \frac{1}{Z}V_{t+1}(x_{t+1})p_t(x_{t+1}),
    \end{equation}
    then 
    \begin{align}
        \pi_t^*(x_t) &= \sum_{x_{t+1}} \pi_{t+1}^*(x_{t+1})\pi_{t|t+1}^*(x_t | x_{t+1}) \\
        &= \sum_{x_{t+1}} \frac{V_{t+1}(x_{t+1})p_{t+1}(x_{t+1})}{Z}\frac{V_t(x_t)p_{t|t+1}(x_t | x_{t+1})}{V_{t+1}(x_{t+1})} \\
        &= \frac{1}{Z}V_t(x_t)p_t(x_t).
    \end{align}
    Proof ends by induction.
    
\end{proof}

\begin{algorithm}
\caption{Multinomial resampling} \label{alg:resample}
\begin{algorithmic}[1]
    \State \textbf{Input: } $\{x^{(i)}\}_{i=1}^N, \{w^{(i)}\}_{i=1}^N$
    \For{$k \gets 1, 2, \cdots, N$}
        \State $i_k \overset{\text{iid}}{\sim} \text{Cat}(i: \frac{w^{(i)}}{\sum_{n=1}^N w^{(n)}})$
        \State $y^{(k)} \gets x^{(i_k)}$
    \EndFor
    \State \textbf{Output: } $\{y^{(k)}\}_{k=1}^N, \{1 \}_{k=1}^N$
\end{algorithmic}
\end{algorithm}

\subsection{Reduced SMC algorithm}

Algorithm \ref{alg:resample} defines the idealized multinomial resampling operator utilized in our analysis. In practice, more efficient variants such as systematic, adaptive, or partial resampling may be employed; we refer the reader to \cite{naesseth2019elements} for a comprehensive review of these methods.

\begin{lemma}[Importance Sampling]\label{IS}
    Let $(x^{(i)}, \tilde{x}^{(i)})\overset{\text{iid}}{\sim} \rho(x, \tilde{x})$, and let $w(x, \tilde{x})\geq 0$ be a weight function such that $w^{(i)} = w(x^{(i)}, \tilde{x}^{(i)})$. As $N \rightarrow \infty$, the output particles $\{y^{(k)}\}_{k=1}^N$ of Algorithm \ref{alg:resample} satisfy:
    \begin{equation}
        y^{(k)} \overset{\text{iid}}{\sim} \pi(x) \propto \sum_{\tilde{x}} w(x, \tilde{x})\rho(x, \tilde{x}).
    \end{equation}
\end{lemma}

\begin{proof}
    For any test function $\phi(x)$, 
    \begin{align}
        \mathbb{E}_{x^{(i)}, \tilde{x}^{(i)}\overset{\text{iid}}{\sim} \rho(x, \tilde{x})}[\phi(y^{(k)})] &= \mathbb{E}_{x^{(i)}, \tilde{x}^{(i)}\overset{\text{iid}}{\sim} \rho(x, \tilde{x})}[ \frac{\frac{1}{N}\sum_{i=1}^N w(x^{(i)}, \tilde{x}^{(i)})\phi(x^{(i)})}{\frac{1}{N}\sum_{n=1}^N w(x^{(n)}, \tilde{x}^{(n)})}].
    \end{align}

    By the law of large numbers, as $N\rightarrow \infty$,
    \begin{align}
        \frac{1}{N}\sum_{n=1}^N w(x^{(n)}, \tilde{x}^{(n)}) &\xrightarrow{\text{a.s.}} \mathbb{E}_{x, \tilde{x} \sim \rho} [w(x, \tilde{x})] \\
        \frac{1}{N}\sum_{i=1}^N w(x^{(i)}, \tilde{x}^{(i)})\phi(x^{(i)}) &\xrightarrow{\text{a.s.}} \mathbb{E}_{x, \tilde{x} \sim \rho}[w(x, \tilde{x})\phi(x)] = \mathbb{E}_{x, \tilde{x} \sim \rho} [w(x, \tilde{x})]\times \mathbb{E}_{x\sim \pi}[\phi(x)].
    \end{align}

    Thus, we get
    \begin{equation}
        \mathbb{E}_{x^{(i)}, \tilde{x}^{(i)}\overset{\text{iid}}{\sim} \rho(x, \tilde{x})}[\phi(y^{(k)})] \overset{N\rightarrow \infty}{\rightarrow} \mathbb{E}_{x \sim \pi}[\phi(x)],
    \end{equation}
    so $y^{(k)} \sim \pi$.

    Lastly, regarding the independence of $\{y^{(k)}\}_{k=1}^N$, the multinomial resampling process draws each $y^{(k)}$ independently given the particle set. As $N \rightarrow \infty$, this process is equivalent to i.i.d. sampling from the target distribution $\pi$, as the empirical distribution converges to the true $\pi$.

\end{proof}

\begin{algorithm}[t]
\caption{Full SMC guidance}
    \label{alg:full_smc}
    \begin{algorithmic}[1]
    \For{$i \gets 1, 2, \cdots, N$}
        \State $x_T^{(i)} \overset{\text{iid}}{\sim} q_T(x_T)$
        \State $x_0^{(i, j)} \overset{\text{iid}}{\sim} p_{0|T}(x_0 | x_T^{(i)}) \quad (j=1, 2, \cdots, J)$
        \State $\hat{V}_T^{(i)} \gets \frac{1}{J}\sum_{j=1}^J R(x_0^{(i, j)})$
        \State $w_T^{(i)} \gets p_T(x_T^{(i)})\times \hat{V}_T^{(i)}/ q_T(x_T^{(i)})$
    \EndFor
    \State $\{x_T^{(i)}\}_{i=1}^N, \{w_T^{(i)}\}_{i=1}^N \gets \text{Resample}(\{x_T^{(i)}\}_{i=1}^N, \{w_T^{(i)}\}_{i=1}^N)$ \label{Tend}
    \For{$t \gets T-1, T-2, \cdots, 0$}
        \For{$i \gets 1, 2, \cdots, N$}
            \State $x_t^{(i)} \overset{\text{iid}}{\sim} q_{t | t+1}(x_t | x_{t+1}^{(i)})$
            \State $x_0^{(i, j)} \overset{\text{iid}}{\sim} p_{0|t}(x_0 | x_t^{(i)}) \quad (j=1, 2, \cdots, J)$
            \State $\hat{V}_t^{(i)} \gets \frac{1}{J}\sum_{j=1}^J R(x_0^{(i, j)})$
            \State $w_t^{(i)} \gets \frac{p_{t|t+1}(x_t^{(i)} | x_{t+1}^{(i)})}{q_{t|t+1}(x_t^{(i)} | x_{t+1}^{(i)})} \times \frac{\hat{V}_t^{(i)}}{\hat{V}_{\lceil t+1 \rceil}^{(i)}}\times w_{\lceil t+1 \rceil}^{(i)}$
        \EndFor
        \State $\{x_t^{(i)}\}_{i=1}^N, \{w_t^{(i)}\}_{i=1}^N \gets \text{Resample}(\{x_t^{(i)}\}_{i=1}^N, \{w_t^{(i)}\}_{i=1}^N)$
        \label{tend}
    \EndFor
    \end{algorithmic}
\end{algorithm}

\begin{algorithm}
\caption{Sparse SMC guidance}
\label{alg:sparsesmc}
\begin{algorithmic}[1]
\For{$i \gets 1, 2, \cdots, N$}
    \State $x_T^{(i)} \overset{\text{iid}}{\sim} q_T(x_T)$
    \State $x_0^{(i, j)} \overset{\text{iid}}{\sim} p_{0|T}(x_0 | x_T^{(i)}) \quad (j=1, 2, \cdots, J)$
    \State $\hat{V}_T^{(i)} \gets \frac{1}{J}\sum_{j=1}^J R(x_0^{(i, j)})$
    \State $w_T^{(i)} \gets p_T(x_T^{(i)})\times \hat{V}_T^{(i)}/ q_T(x_T^{(i)})$
\EndFor
\State $\{x_T^{(i)}\}_{i=1}^N, \{w_T^{(i)}\}_{i=1}^N \gets \text{Resample}(\{x_T^{(i)}\}_{i=1}^N, \{w_T^{(i)}\}_{i=1}^N)$ \label{Tend}
\For{$t \gets T-1, T-2, \cdots, 0$}
    \If{$t \in \mathcal{G}$}
        \For{$i \gets 1, 2, \cdots, N$}
            \State $x_t^{(i)} \overset{\text{iid}}{\sim} q_{t | t+1}(x_t | x_{t+1}^{(i)})$
            \State $x_0^{(i, j)} \overset{\text{iid}}{\sim} p_{0|t}(x_0 | x_t^{(i)}) \quad (j=1, 2, \cdots, J)$
            \State $\hat{V}_t^{(i)} \gets \frac{1}{J}\sum_{j=1}^J R(x_0^{(i, j)})$
            \State $w_t^{(i)} \gets \frac{p_{t|t+1}(x_t^{(i)} | x_{t+1}^{(i)})}{q_{t|t+1}(x_t^{(i)} | x_{t+1}^{(i)})} \times \frac{\hat{V}_t^{(i)}}{\hat{V}_{\lceil t+1 \rceil}^{(i)}}\times w_{\lceil t+1 \rceil}^{(i)}$
        \EndFor
        \State $\{x_t^{(i)}\}_{i=1}^N, \{w_t^{(i)}\}_{i=1}^N \gets \text{Resample}(\{x_t^{(i)}\}_{i=1}^N, \{w_t^{(i)}\}_{i=1}^N)$
        \label{tend}
    \Else
        \For{$i \gets 1, 2, \cdots, N$}
            \State $x_t^{(i)} \overset{\text{iid}}{\sim} p_{t | t+1}(x_t | x_{t+1}^{(i)})$
        \EndFor
    \EndIf
\EndFor
\end{algorithmic}
\end{algorithm}

\begin{theorem}
\label{SMCthm}
    Say $N, J \rightarrow \infty$. At every guided time $t\in \mathcal{G}\cup \{T\}$, the particles in Algorithm \ref{alg:sparsesmc} are distributed according to the ideal guided target distribution, i.e., $x_t^{(i)} \overset{\text{iid}}{\sim} \pi_t^* \quad (i=1, 2, \cdots, N)$.
\end{theorem}

\begin{proof}
    Since $J\rightarrow \infty$, the Monte Carlo value estimate $\hat{V}_t^{(i)}$ converges almost surely to the true value function $V_t(x_t^{(i)})$ by the strong law of large numbers.

    Base Case: At $t=T$, the particles are sampled from $q_T$ and weighted by $w_T \propto p_T(x_T)V_T(x_T)/q_T(x_T)$. By Lemma \ref{IS}, after resampling at Line \ref{Tend}, we have $x_T^{(i)}\overset{\text{iid}}{\sim} \pi_T^*(x_T) \propto p_T(x_T)V_T(x_T)$.

    Inductive Step: Assume that $x_{\lceil t+1 \rceil}^{(i)} \overset{\text{iid}}{\sim} \pi_{\lceil t+1 \rceil}^*(x_{\lceil t+1 \rceil})$.

    Then, just before the line \ref{tend},
    \begin{equation}
        (x_t^{(i)}, x_{t+1}^{(i)}, x_{\lceil t+1 \rceil}^{(i)}) \overset{\text{iid}}{\sim} q_{t | t+1}(x_t | x_{t+1})\cdot p_{t+1 | \lceil t+1 \rceil}(x_{t+1} | x_{\lceil t+1 \rceil})\cdot \frac{p_{\lceil t+1 \rceil}(x_{\lceil t+1 \rceil})V_{\lceil t+1 \rceil}(x_{\lceil t+1 \rceil})}{Z}.
    \end{equation}
    Just after the line \ref{tend},
    \begin{align}
        (x_t^{(i)}, x_{t+1}^{(i)}, x_{\lceil t+1 \rceil}^{(i)}) &\overset{\text{iid}}{\sim} q_{t | t+1}(x_t | x_{t+1}) p_{t+1 | \lceil t+1 \rceil}(x_{t+1} | x_{\lceil t+1 \rceil}) \frac{p_{\lceil t+1 \rceil}(x_{\lceil t+1 \rceil})V_{\lceil t+1 \rceil}(x_{\lceil t+1 \rceil})}{Z} \nonumber\\ 
        &\times \frac{p_{t|t+1}(x_t|x_{t+1})}{q_{t| t+1}(x_t | x_{t+1})}\frac{V_t(x_t)}{V_{\lceil t+1 \rceil}(x_{\lceil t+1 \rceil})} \\
        &= \frac{1}{Z}p_{t, t+1, \lceil t+1 \rceil}(x_t, x_{t+1}, x_{\lceil t+1 \rceil}) V_t(x_t).
    \end{align}
    By Lemma \ref{IS}, we may conclude that
    \begin{equation}
        x_t^{(i)} \overset{\text{iid}}{\sim} \pi_t^*(x_t).
    \end{equation}

    The proof follows by induction for all $t\in \mathcal{G}$.
\end{proof}

\subsection{Probability path distance}\label{TDsec}

We first record a simple inequality.
\begin{lemma}\label{simpineq}
Random variable $X \in \mathbb{R}$ satisfies $0 \leq X \leq M$ and $\mu = \mathbb{E}[X]$. Then, 
\begin{equation}
    \mathbb{E}[|X-\mu|] \leq 2\mu\left(1-\frac{\mu}{M}\right).
\end{equation}
\end{lemma}

\begin{proof}
Since
\begin{align*}
    0 = \mathbb{E}[X-\mu] = \mathbb{E}[(X-\mu)\mathbf{1}(X>\mu) + (X - \mu)\mathbf{1}(X\leq \mu)],
\end{align*}
we get the following identity:
\begin{equation}
    \mathbb{E}[(X-\mu)\mathbf{1}(X>\mu)] = \mathbb{E}[(\mu-X)\mathbf{1}(X \leq \mu)].
\end{equation}

\noindent
\begin{minipage}{0.55\textwidth}
    Thus,
    \begin{align*}
        \mathbb{E}[|X-\mu|] &= \mathbb{E}[(X-\mu)\mathbf{1}(X>\mu)] \\ &+ \mathbb{E}[(\mu-X)\mathbf{1}(X \leq \mu)] \\
        &= 2\mathbb{E}[(\mu-X)\mathbf{1}(X \leq \mu)] \\
        &\leq 2\mathbb{E}[\mu(1-X/M)] = 2\mu(1-\mu/M).
    \end{align*}
    Applying the expectation to both sides:
    \begin{align*}
        \mathbb{E}[|X-\mu|] &\leq 2\mathbb{E}\left[\mu\left(1-\frac{X}{M}\right)\right] \\
        &= 2\mu\left(1-\frac{\mathbb{E}[X]}{M}\right) = 2\mu\left(1-\frac{\mu}{M}\right).
    \end{align*}
\end{minipage}
\hfill
\begin{minipage}{0.4\textwidth}
    \centering
    \includegraphics[width=\linewidth]{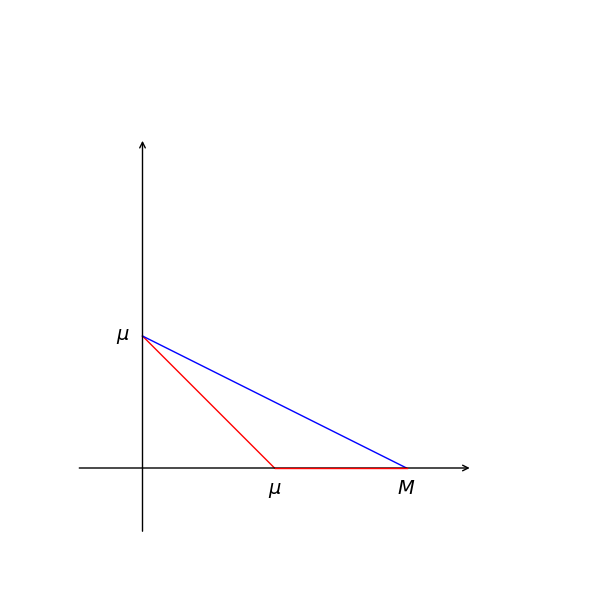}
    \captionof{figure}{Illustration of $(\mu-X)\mathbf{1}(X \leq \mu) \leq \mu(1-\frac{X}{M})$}
    \label{fig:geo_proof}
\end{minipage}
\end{proof}

\begin{proof}[Proof of Theorem~\ref{UTDthm}]
    If $t\in \mathcal{G}$, then $\pi_t^* = \pi_t^{\mathcal{G}}$, so
    \begin{equation}
        \mathbf{TD}(\mathcal{G}) = \frac{1}{T}\sum_{t\in \mathcal{T}-\mathcal{G}} \|\pi_t^* - \pi_t^{\mathcal{G}}\|_1.
    \end{equation}

    By Lemma \ref{bellman},
    \begin{equation}
        V_{\ceilG{t}}(x_{\ceilG{t}}) = \mathbb{E}_{x_t \sim p_{t | \lceil t\rceil}(\cdot | x_{\ceilG{t}})}[V_t(x_t)].
    \end{equation}

    Thus,
    \begin{align}
        \|\pi_t^* - \pi_t^{\mathcal{G}}\|_1 &= \sum_{x_t \in \mathcal{X}} |\pi_t^*(x_t) - \sum_{x_{\ceilG{t}} \in \mathcal{X}} \pi_{\ceilG{t}}^*(x_{\ceilG{t}})p_{t|\ceilG{t}}(x_t | x_{\ceilG{t}})| \\
        &= \frac{1}{Z}\sum_{x_t \in \mathcal{X}}|\sum_{x_{\ceilG{t}} \in \mathcal{X}} p_{\ceilG{t}}(x_{\ceilG{t}})p_{t|\ceilG{t}}(x_t | x_{\ceilG{t}})(V_t(x_t)-V_{\ceilG{t}}(x_{\ceilG{t}}))| \\
        &\leq \frac{1}{Z}\mathbb{E}_{x_{\ceilG{t}}\sim p_{\ceilG{t}}, x_t\sim p_{t|\ceilG{t}}(\cdot | x_{\lceil t \rceil})} |V_t(x_t)-V_{\ceilG{t}}(x_{\ceilG{t}})|. \label{mediate}
    \end{align}

    Applying Lemma \ref{simpineq} with $X = V_t(x_t)$ and $\mu=V_{\ceilG{t}}(x_{\ceilG{t}})$ gives
    \begin{equation}
        \mathbb{E}_{x_t \sim p_{t|\ceilG{t}}(\cdot | x_{\ceilG{t}})}[|V_t(x_t) - V_{\ceilG{t}}(x_{\ceilG{t}})|] \leq 2 V_{\ceilG{t}}(x_{\ceilG{t}})(1-\frac{V_{\ceilG{t}}(x_{\ceilG{t}})}{R_{\text{max}}}).
    \end{equation}

    Therefore,
    \begin{align}
        \|\pi_t^* - \pi_t^{\mathcal{G}}\|_1 &\leq \frac{2}{Z}\mathbb{E}_{x_{\ceilG{t}}\sim p_{\ceilG{t}}}[V_{\ceilG{t}}(x_{\ceilG{t}})(1-\frac{V_{\ceilG{t}}(x_{\ceilG{t}})}{R_{\text{max}}})] \\
        &= 2(1-\frac{1}{R_{\text{max}}}\mathbb{E}_{x_{\ceilG{t}}\sim \pi^*_{\ceilG{t}}}[V_{\ceilG{t}}(x_{\ceilG{t}})]).
    \end{align}

\end{proof}

\subsection{DP algorithm}\label{DP}

We aim to find $\mathcal{G}$ that maximizes 
\begin{equation}
    \hat{\mathbf{S}}(\mathcal{G}) := \sum_{t\in \mathcal{T}\setminus\mathcal{G}} \hat{V}_{\ceilG{t}}
\end{equation}
where $T' = |\mathcal{G}|$ is a given constant.

Define
\begin{equation}
    dp[s][i] := \max_{\mathcal{G}'} \sum_{t \in \mathcal{T}\setminus\mathcal{G}'; t> s} \hat{V}_{\ceilG{t}}
\end{equation}
where
\begin{equation}
    \mathcal{G}' = \{s = t_i< \cdots <t_1<T\}.
\end{equation}

We can observe the recurrence relation
\begin{align}
    dp[s][1] &= (T-s-1)\hat{V}_T \\
    dp[s][i] &= \max_{s < s' \leq T-i+1} [dp[s'][i-1] + (s'-s-1)\hat{V}_{s'}]. \quad (i\geq 2)
\end{align}

Thus, Algorithm \ref{alg:dp} can compute the optimal $\mathcal{G}$.

\begin{algorithm}
\caption{DP algorithm} \label{alg:dp}
\begin{algorithmic}[1]
    \State \text{Initialize: } $dp[s][i] \gets 0$ for all $s, i$
    \For{$s\gets T'-1, T', \cdots, T-1$}
        \State $dp[s][1] \gets (T-s-1)\hat{V}_T$
    \EndFor
    \For{$i\gets 2, 3, \cdots, T'$}
        \For{$s \gets T'-i, \cdots, T-i$}
            \For{$s' \gets s+1$ to $T-i+1$}
                \State $c \gets dp[s'][i-1] + (s'-s-1)\hat{V}_{s'}$
                \If{$c > dp[s][i]$}
                    \State $dp[s][i] \gets c$
                    \State $parent[s][i] \gets s'$
                \EndIf
            \EndFor
        \EndFor
    \EndFor
    \State $t_{T'} \gets \text{argmax}_s (dp[s][T'] + s\hat{V}_s)$
    \For{$j\gets T'-1, T'-2, \cdots, 1$}
        \State $t_j \gets parent[t_{j+1}][j+1]$
    \EndFor
    \State \Return $\mathcal{G} = \{t_1, t_2, \cdots, t_{T'}\}$
\end{algorithmic}
\end{algorithm}

    
\section{Experimental Details}\label{Expdetail}

\subsection{Experimental Setup}

The Protein Fitness experiment was conducted on a single PRO 6000 GPU,
while all other experiments were conducted on a single NVIDIA A6000 GPU.
Unlike SMC-based evaluation protocols that generate $N$ candidates but retain
only the highest-reward sample, we retain all generated samples to evaluate
both alignment and sample diversity.

For Binary MNIST, we use the masked diffusion model (MDM) from
\citep{sahoo2024simple} and follow the implementation of
\citep{ou2025inference}.
For Toxic Text Generation, we use the pretrained MDLM from
\citep{sahoo2024simple} together with the reward model from
\citep{logacheva2022paradetox}, following the implementation of
\citep{ou2025inference}.
For CIFAR-10, we use the MDM and implementation described in the corresponding
experimental setup.
For DNA generation, we follow the experimental setup of
\citep{ou2025inference}.
For Protein Fitness, we use the pretrained MDLM from
\citep{sahoo2024simple} together with a fitness oracle as the reward model.
The resampling configurations for each experiment are summarized in
Table~\ref{tab:resampling}.

\begin{table}[]
    \centering
    \caption{
    Resampling configurations for each experiment.
    Adaptive resampling is triggered when the normalized Effective Sample Size
    (ESS$/N$) falls below the specified threshold.
    Partial resampling denotes the number of particles replaced at each
    resampling step.
    }
    \begin{tabular}{lccc}
    \toprule
    Experiment & ESS threshold & Partial resampling & Systematic resampling \\
    \midrule
    Binary MNIST     & 0.95 & $\lfloor N/2 \rfloor$ & Yes \\
    Toxic Text       & 0.98 & $N$                    & Yes \\
    CIFAR-10          & 0.95 & $\lfloor N/2 \rfloor$ & Yes \\
    DNA Design       & 0.98 & $\lfloor N/2 \rfloor$ & Yes \\
    Protein Fitness  & 0.98 & $\lfloor N/2 \rfloor$ & Yes \\
    Continuous MNIST & 0.98 & $\lfloor N/2 \rfloor$ & Yes \\
    \bottomrule
    \end{tabular}
    \label{tab:resampling}
\end{table}

\subsection{Details on main experiment}\label{mainexperimentappendix}

\subsubsection{Details on Binary MNIST experiment}\label{binarymnistappendix}

The warmup phase took 5.78s for SMC-base, and 24.7s for SMC-grad. See Figure \ref{fig:mnistbasegraph}, \ref{fig:mnistbasebudget}, \ref{fig:mnistgradbudget} and Table \ref{tab:binarymnistbase}, \ref{tab:binarymnistgrad}.

\begin{figure}[htbp]
  \centering
  \begin{minipage}[t]{0.32\textwidth}
    \centering
    \vspace{0pt} 
    \includegraphics[trim={0 0 0 0}, clip, width=\textwidth]{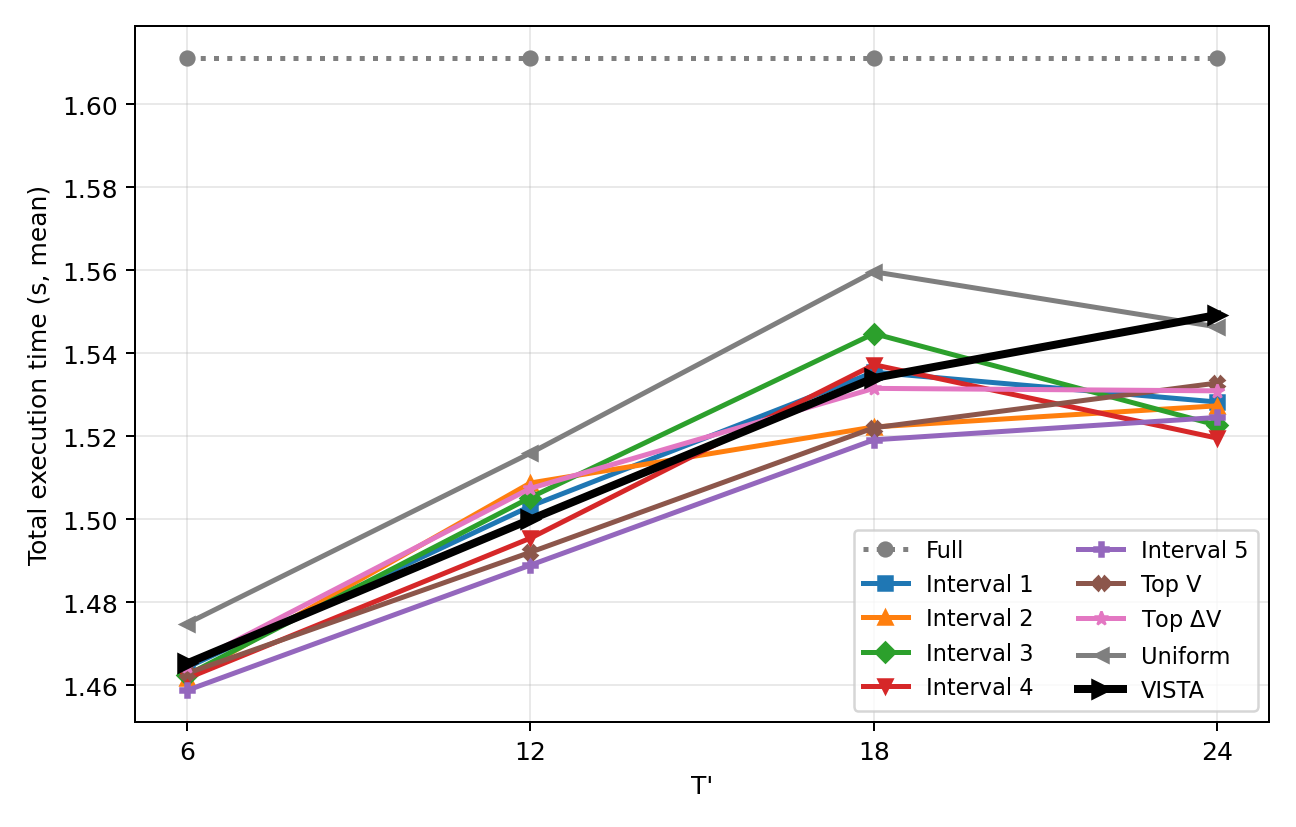}
    \label{fig:mnisttime}
  \end{minipage}
  \hfill 
  \begin{minipage}[t]{0.32\textwidth}
    \centering
    \vspace{0pt} 
    \includegraphics[trim={0 0 0 0}, clip, width=\textwidth]{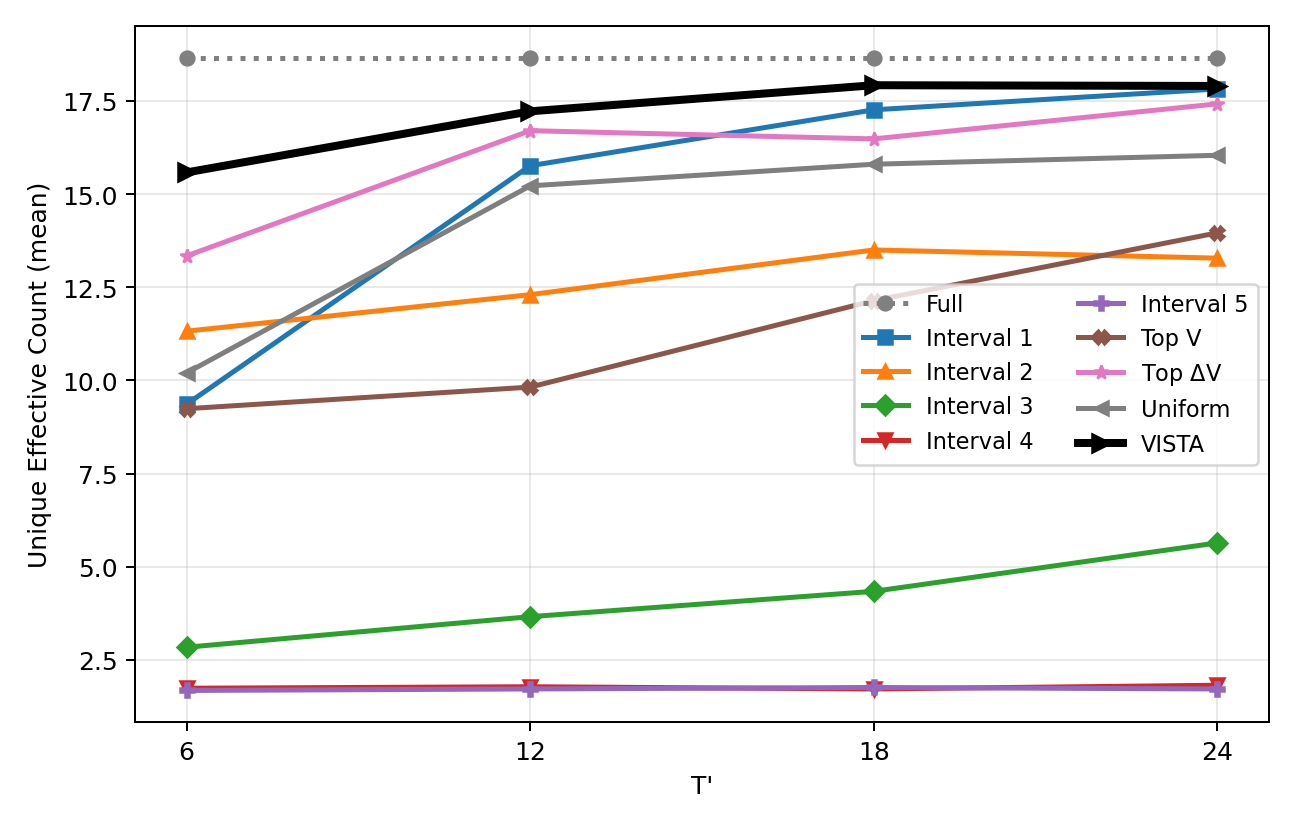}
    \label{fig:mnistuval}
  \end{minipage}
  \hfill 
  \begin{minipage}[t]{0.32\textwidth}
    \centering
    \vspace{0pt} 
    \includegraphics[trim={0 0 0 0}, clip, width=\textwidth]{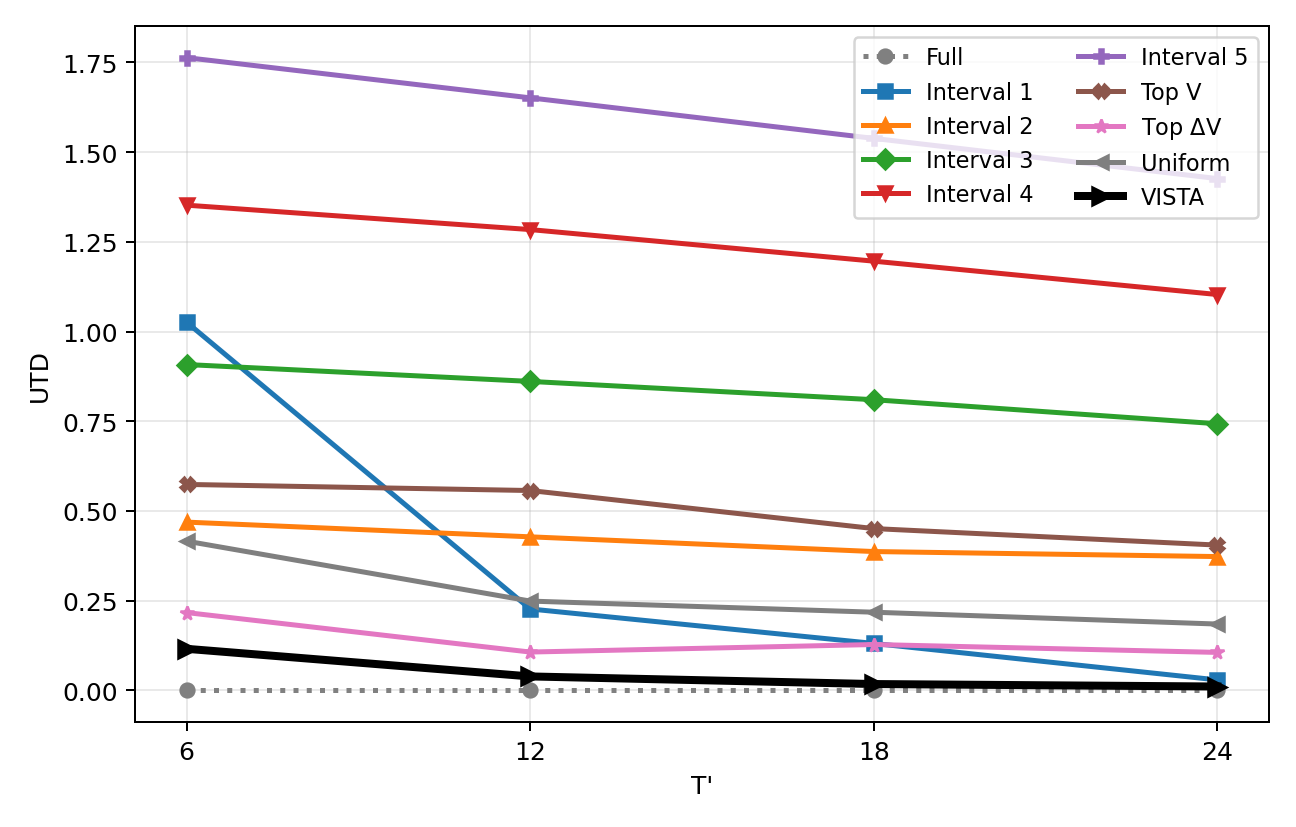}
    \label{fig:mnistutd}
  \end{minipage}
  \caption{Binary MNIST (SMC-base). 
    Execution time, unique valid samples, and UTD as the guidance budget $T'$ varies.}
  \label{fig:mnistbasegraph}
\end{figure}

\begin{table*}[t]
\centering
\begin{minipage}[t]{0.485\textwidth}
\centering
\small
\caption{Binary MNIST (SMC-base): execution time, unique valid, and UTD.}
\label{tab:binarymnistbase}
\resizebox{\linewidth}{!}{
\setlength{\tabcolsep}{4.5pt}
\renewcommand{\arraystretch}{0.80}
\begin{tabular}{llccc}
\toprule
\textbf{$T'$} & \textbf{Policy} & {\textbf{Time [s]} $\downarrow$} & {\textbf{Unique Valid} $\uparrow$} & {\textbf{UTD} $\downarrow$} \\
\midrule
\multirow{1}{*}{100}  & Full  & $1.61 \pm 0.01$  & $\mathbf{18.66 \pm 1.78}$  & {--} \\
\midrule
\multirow{9}{*}{6}  & Interval 1  & $1.46 \pm 0.02$  & $9.36 \pm 3.53$  & $1.026$ \\
  & Interval 2  & $1.46 \pm 0.02$  & $11.32 \pm 6.19$  & $0.469$ \\
  & Interval 3  & $1.46 \pm 0.02$  & $2.84 \pm 2.61$  & $0.908$ \\
  & Interval 4  & $1.46 \pm 0.02$  & $1.74 \pm 1.05$  & $1.352$ \\
  & Interval 5  & $1.46 \pm 0.02$  & $1.68 \pm 1.04$  & $1.763$ \\
  & Top-$V$  & $1.46 \pm 0.02$  & $9.24 \pm 5.84$  & $0.574$ \\
  & Top-$\Delta V$  & $1.46 \pm 0.02$  & $13.34 \pm 4.12$  & $0.217$ \\
  & Uniform  & $1.47 \pm 0.09$  & $10.20 \pm 3.12$  & $0.416$ \\
  & VISTA  & $1.47 \pm 0.02$  & $\mathbf{15.58 \pm 2.48}$  & $\mathbf{0.116}$ \\
\midrule
\multirow{9}{*}{12}  & Interval 1  & $1.50 \pm 0.03$  & $15.76 \pm 3.54$  & $0.227$ \\
  & Interval 2  & $1.51 \pm 0.11$  & $12.30 \pm 6.21$  & $0.428$ \\
  & Interval 3  & $1.51 \pm 0.10$  & $3.66 \pm 3.47$  & $0.861$ \\
  & Interval 4  & $1.50 \pm 0.04$  & $1.78 \pm 1.09$  & $1.284$ \\
  & Interval 5  & $1.49 \pm 0.02$  & $1.72 \pm 1.05$  & $1.651$ \\
  & Top-$V$  & $1.49 \pm 0.02$  & $9.82 \pm 5.94$  & $0.557$ \\
  & Top-$\Delta V$  & $1.51 \pm 0.07$  & $16.70 \pm 2.90$  & $0.107$ \\
  & Uniform  & $1.52 \pm 0.12$  & $15.22 \pm 3.97$  & $0.249$ \\
  & VISTA  & $1.50 \pm 0.02$  & $\mathbf{17.22 \pm 3.45}$  & $\mathbf{3.91 \times 10^{-2}}$ \\
\midrule
\multirow{9}{*}{18}  & Interval 1  & $1.54 \pm 0.02$  & $17.26 \pm 3.33$  & $0.131$ \\
  & Interval 2  & $1.52 \pm 0.02$  & $13.50 \pm 5.95$  & $0.387$ \\
  & Interval 3  & $1.54 \pm 0.11$  & $4.34 \pm 4.03$  & $0.810$ \\
  & Interval 4  & $1.54 \pm 0.08$  & $1.72 \pm 1.09$  & $1.196$ \\
  & Interval 5  & $1.52 \pm 0.02$  & $1.76 \pm 1.08$  & $1.538$ \\
  & Top-$V$  & $1.52 \pm 0.02$  & $12.14 \pm 6.14$  & $0.451$ \\
  & Top-$\Delta V$  & $1.53 \pm 0.03$  & $16.48 \pm 4.08$  & $0.128$ \\
  & Uniform  & $1.56 \pm 0.16$  & $15.80 \pm 3.60$  & $0.218$ \\
  & VISTA  & $1.53 \pm 0.04$  & $\mathbf{17.92 \pm 2.47}$  & $\mathbf{1.79 \times 10^{-2}}$ \\
\midrule
\multirow{9}{*}{24}  & Interval 1  & $1.53 \pm 0.02$  & $17.82 \pm 2.60$  & $2.93 \times 10^{-2}$ \\
  & Interval 2  & $1.53 \pm 0.03$  & $13.28 \pm 5.82$  & $0.373$ \\
  & Interval 3  & $1.52 \pm 0.02$  & $5.64 \pm 4.35$  & $0.743$ \\
  & Interval 4  & $1.52 \pm 0.02$  & $1.82 \pm 1.16$  & $1.103$ \\
  & Interval 5  & $1.52 \pm 0.03$  & $1.72 \pm 1.05$  & $1.426$ \\
  & Top-$V$  & $1.53 \pm 0.10$  & $13.96 \pm 5.61$  & $0.405$ \\
  & Top-$\Delta V$  & $1.53 \pm 0.02$  & $17.42 \pm 3.00$  & $0.106$ \\
  & Uniform  & $1.55 \pm 0.15$  & $16.04 \pm 3.69$  & $0.185$ \\
  & VISTA  & $1.55 \pm 0.10$  & $\mathbf{17.90 \pm 3.05}$  & $\mathbf{1.10 \times 10^{-2}}$ \\
\bottomrule
\end{tabular}
}
\end{minipage}
\hfill
\begin{minipage}[t]{0.485\textwidth}
\centering
\small
\caption{Binary MNIST (SMC-grad): execution time, unique valid, and UTD.}
\label{tab:binarymnistgrad}
\resizebox{\linewidth}{!}{
\setlength{\tabcolsep}{4.5pt}
\renewcommand{\arraystretch}{0.80}
\begin{tabular}{llccc}
\toprule
\textbf{$T'$} & \textbf{Policy} & {\textbf{Time [s]} $\downarrow$} & {\textbf{Unique Valid} $\uparrow$} & {\textbf{UTD} $\downarrow$} \\
\midrule
\multirow{1}{*}{100}  & Full  & $7.13 \pm 0.28$  & $\mathbf{18.78 \pm 1.73}$  & {--} \\
\midrule
\multirow{9}{*}{6}  & Interval 1  & $2.12 \pm 0.48$  & $10.14 \pm 3.93$  & $0.795$ \\
  & Interval 2  & $2.04 \pm 0.36$  & $11.68 \pm 6.11$  & $0.461$ \\
  & Interval 3  & $2.07 \pm 0.38$  & $2.96 \pm 2.57$  & $0.884$ \\
  & Interval 4  & $2.00 \pm 0.33$  & $1.74 \pm 1.03$  & $1.314$ \\
  & Interval 5  & $1.96 \pm 0.30$  & $1.70 \pm 1.07$  & $1.763$ \\
  & Top-$V$  & $2.09 \pm 0.44$  & $5.30 \pm 4.12$  & $0.714$ \\
  & Top-$\Delta V$  & $2.06 \pm 0.44$  & $14.14 \pm 3.71$  & $0.114$ \\
  & Uniform  & $2.03 \pm 0.42$  & $10.06 \pm 3.75$  & $0.355$ \\
  & VISTA  & $2.10 \pm 0.47$  & $\mathbf{15.72 \pm 2.85}$  & $\mathbf{6.69 \times 10^{-2}}$ \\
\midrule
\multirow{9}{*}{12}  & Interval 1  & $2.78 \pm 0.49$  & $16.22 \pm 2.97$  & $5.70 \times 10^{-2}$ \\
  & Interval 2  & $2.78 \pm 0.48$  & $13.10 \pm 6.36$  & $0.418$ \\
  & Interval 3  & $2.78 \pm 0.48$  & $3.34 \pm 2.94$  & $0.827$ \\
  & Interval 4  & $2.79 \pm 0.47$  & $1.78 \pm 1.07$  & $1.239$ \\
  & Interval 5  & $2.74 \pm 0.48$  & $1.70 \pm 1.02$  & $1.651$ \\
  & Top-$V$  & $2.71 \pm 0.52$  & $5.18 \pm 4.41$  & $0.714$ \\
  & Top-$\Delta V$  & $2.70 \pm 0.51$  & $16.54 \pm 2.84$  & $3.19 \times 10^{-2}$ \\
  & Uniform  & $2.72 \pm 0.53$  & $15.40 \pm 4.45$  & $0.175$ \\
  & VISTA  & $2.68 \pm 0.52$  & $\mathbf{16.82 \pm 3.33}$  & $\mathbf{9.72 \times 10^{-3}}$ \\
\midrule
\multirow{9}{*}{18}  & Interval 1  & $3.37 \pm 0.15$  & $17.68 \pm 2.79$  & $1.26 \times 10^{-2}$ \\
  & Interval 2  & $3.39 \pm 0.13$  & $14.04 \pm 5.99$  & $0.379$ \\
  & Interval 3  & $3.36 \pm 0.21$  & $4.28 \pm 3.61$  & $0.770$ \\
  & Interval 4  & $3.35 \pm 0.23$  & $1.72 \pm 1.03$  & $1.164$ \\
  & Interval 5  & $3.33 \pm 0.27$  & $1.70 \pm 1.02$  & $1.538$ \\
  & Top-$V$  & $3.37 \pm 0.26$  & $5.48 \pm 4.55$  & $0.714$ \\
  & Top-$\Delta V$  & $3.31 \pm 0.31$  & $\mathbf{18.44 \pm 1.67}$  & $6.34 \times 10^{-3}$ \\
  & Uniform  & $3.35 \pm 0.42$  & $17.44 \pm 2.65$  & $0.126$ \\
  & VISTA  & $3.35 \pm 0.24$  & $18.32 \pm 1.96$  & $\mathbf{2.25 \times 10^{-3}}$ \\
\midrule
\multirow{9}{*}{24}  & Interval 1  & $3.23 \pm 0.31$  & $18.44 \pm 2.04$  & $1.03 \times 10^{-2}$ \\
  & Interval 2  & $3.28 \pm 0.27$  & $15.02 \pm 5.06$  & $0.358$ \\
  & Interval 3  & $3.27 \pm 0.28$  & $5.56 \pm 4.47$  & $0.714$ \\
  & Interval 4  & $3.25 \pm 0.30$  & $1.82 \pm 1.35$  & $1.070$ \\
  & Interval 5  & $3.19 \pm 0.30$  & $1.70 \pm 1.02$  & $1.426$ \\
  & Top-$V$  & $3.30 \pm 0.30$  & $5.50 \pm 4.22$  & $0.714$ \\
  & Top-$\Delta V$  & $3.24 \pm 0.33$  & $18.42 \pm 1.69$  & $7.43 \times 10^{-3}$ \\
  & Uniform  & $3.24 \pm 0.40$  & $17.32 \pm 2.39$  & $0.109$ \\
  & VISTA  & $3.23 \pm 0.30$  & $\mathbf{18.54 \pm 2.15}$  & $\mathbf{1.43 \times 10^{-3}}$ \\
\bottomrule
\end{tabular}
}
\end{minipage}
\end{table*}

\begin{figure}[htbp]
    \centering
    \begin{subfigure}[b]{0.48\textwidth}
        \centering
        \includegraphics[width=\linewidth]{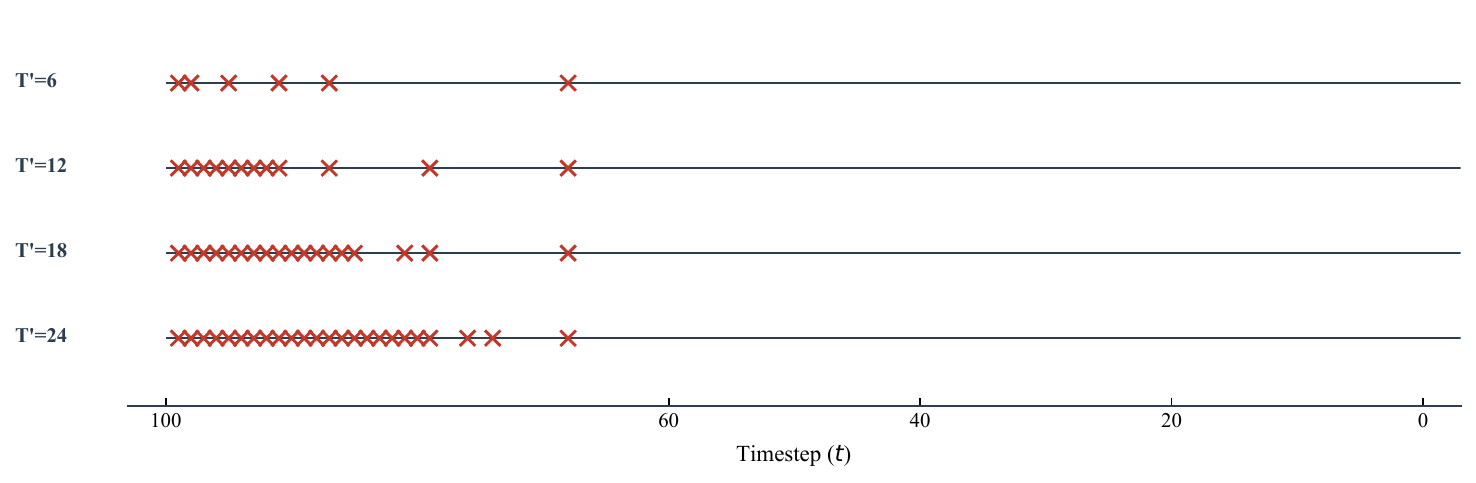}
        \caption{SMC-base}
        \label{fig:mnistbasebudget}
    \end{subfigure}
    \hfill 
    \begin{subfigure}[b]{0.48\textwidth}
        \centering
        \includegraphics[width=\linewidth]{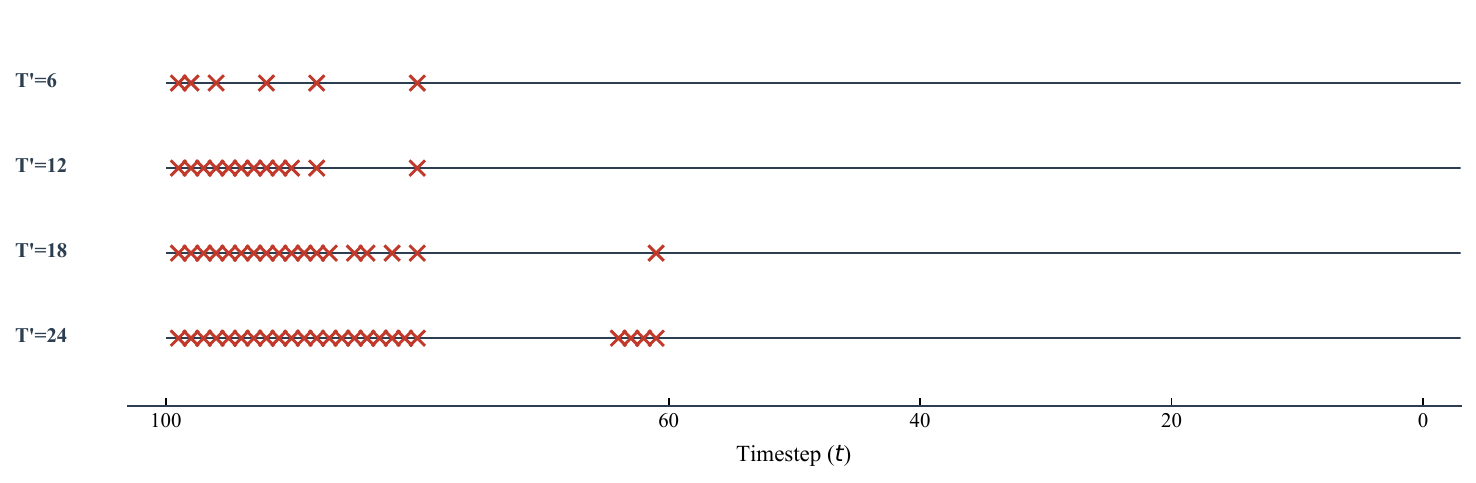}
        \caption{SMC-grad}
        \label{fig:mnistgradbudget}
    \end{subfigure}
    
    \caption{\method schedules $\mathcal{G}$ for varying $T'$ on Binary MNIST}
    \label{fig:mnist_budget_comparison}
\end{figure}

\subsubsection{Details on Toxic Text Generation}\label{toxictextappendix}

See Table \ref{tab:toxictextgradappendix} and Figure \ref{fig:toxictextbudget}.

\begin{table}[t]
\centering
\scriptsize
\caption{Toxic text generation (SMC-grad).}
\label{tab:toxictextgradappendix}
\resizebox{0.85\linewidth}{!}{ 
\setlength{\tabcolsep}{3pt} 
\renewcommand{\arraystretch}{0.95} 
\begin{tabular}{llccc}
\toprule
\textbf{$T'$} & \textbf{Policy} & {\textbf{Time [s]} $\downarrow$} & {\textbf{Unique Toxic Ratio} $\uparrow$} & {$\hat{S}(\mathcal{G})$ $\uparrow$} \\
\midrule
\multirow{1}{*}{100}  & Full  & $204.23 \pm 0.65$  & $\mathbf{1.00 \pm 0.02}$  & $\mathbf{0}$ \\
\midrule
\multirow{9}{*}{5}  & Interval 1  & $14.10 \pm 0.20$  & $0.74 \pm 0.34$  & $68.97$ \\
  & Interval 2  & $15.08 \pm 0.32$  & $0.73 \pm 0.28$  & $55.61$ \\
  & Interval 3  & $15.01 \pm 0.25$  & $0.72 \pm 0.35$  & $34.85$ \\
  & Interval 4  & $15.15 \pm 0.35$  & $0.69 \pm 0.42$  & $23.60$ \\
  & Interval 5  & $14.34 \pm 0.25$  & $0.05 \pm 0.17$  & $5.16$ \\
  & Top-$V$  & $16.01 \pm 0.27$  & $0.16 \pm 0.30$  & $10.20$ \\
  & Top-$\Delta V$  & $15.89 \pm 0.27$  & $0.80 \pm 0.33$  & $73.73$ \\
  & Uniform  & $16.97 \pm 0.27$  & $0.77 \pm 0.34$  & $69.26$ \\
  & VISTA  & $17.58 \pm 0.25$  & $\mathbf{0.81 \pm 0.35}$  & $\mathbf{78.73}$ \\
\midrule
\multirow{9}{*}{10}  & Interval 1  & $24.09 \pm 0.24$  & $0.78 \pm 0.30$  & $64.87$ \\
  & Interval 2  & $25.07 \pm 0.28$  & $0.78 \pm 0.31$  & $53.61$ \\
  & Interval 3  & $25.03 \pm 0.24$  & $0.74 \pm 0.28$  & $32.93$ \\
  & Interval 4  & $25.01 \pm 0.28$  & $0.73 \pm 0.37$  & $22.08$ \\
  & Interval 5  & $24.31 \pm 0.27$  & $0.35 \pm 0.41$  & $4.89$ \\
  & Top-$V$  & $25.09 \pm 0.28$  & $0.27 \pm 0.34$  & $5.73$ \\
  & Top-$\Delta V$  & $29.93 \pm 0.24$  & $0.90 \pm 0.21$  & $71.12$ \\
  & Uniform  & $30.94 \pm 0.21$  & $0.94 \pm 0.09$  & $70.02$ \\
  & VISTA  & $28.28 \pm 0.22$  & $\mathbf{0.95 \pm 0.11}$  & $\mathbf{74.92}$ \\
\midrule
\multirow{9}{*}{15}  & Interval 1  & $33.98 \pm 0.23$  & $0.82 \pm 0.20$  & $63.30$ \\
  & Interval 2  & $34.82 \pm 0.20$  & $0.87 \pm 0.18$  & $50.47$ \\
  & Interval 3  & $34.76 \pm 0.18$  & $0.89 \pm 0.18$  & $32.61$ \\
  & Interval 4  & $34.76 \pm 0.24$  & $0.78 \pm 0.34$  & $20.14$ \\
  & Interval 5  & $34.09 \pm 0.22$  & $0.58 \pm 0.42$  & $4.62$ \\
  & Top-$V$  & $35.02 \pm 0.27$  & $0.50 \pm 0.42$  & $5.43$ \\
  & Top-$\Delta V$  & $43.07 \pm 0.21$  & $0.96 \pm 0.11$  & $67.84$ \\
  & Uniform  & $44.91 \pm 0.19$  & $0.94 \pm 0.19$  & $64.96$ \\
  & VISTA  & $38.16 \pm 0.24$  & $\mathbf{0.99 \pm 0.03}$  & $\mathbf{70.86}$ \\
\bottomrule
\end{tabular}
}
\end{table}

\begin{figure}
    \centering
    \includegraphics[width=\linewidth]{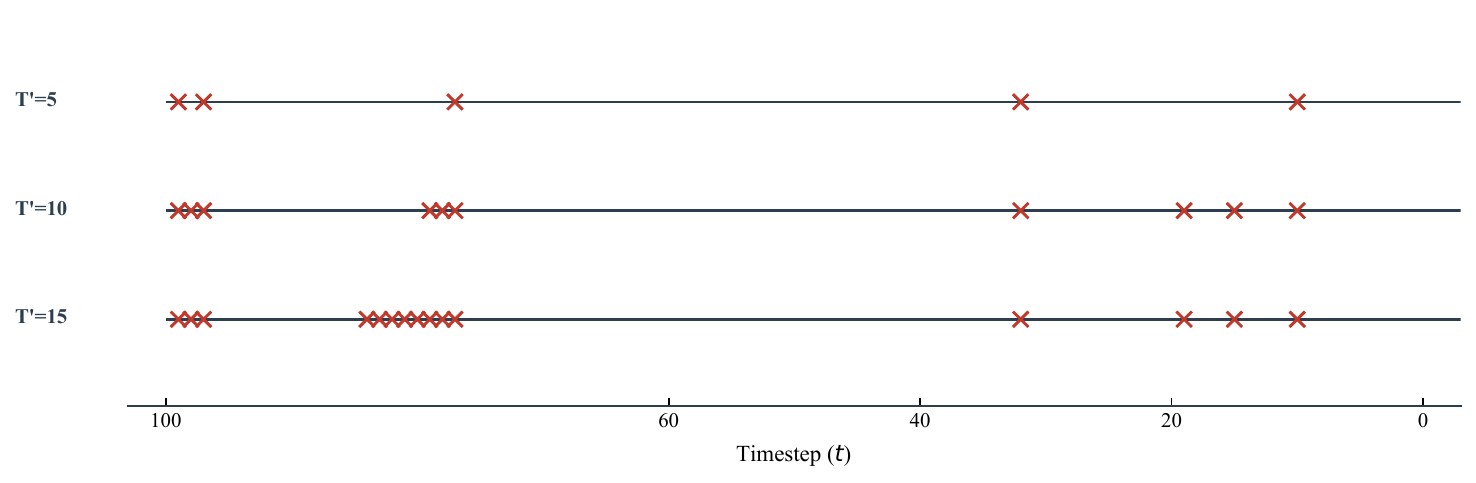}
    \caption{\method schedules $\mathcal{G}$ for varying $T'$ (Toxic Text, SMC-grad)}
    \label{fig:toxictextbudget}
\end{figure}

\subsubsection{Details on CIFAR-10 Image Generation}\label{CIFAR10appendix}

See Table \ref{tab:cifar10-time-success-score} and Figure \ref{fig:cifar10budget}.

\begin{table}[t]
\centering
\small
\caption{CIFAR-10 class-conditional generation: execution time, unique success count, and $\hat{S}(\mathcal{G})$.}
\label{tab:cifar10-time-success-score}
\resizebox{0.85\linewidth}{!}{
\setlength{\tabcolsep}{3pt}
\renewcommand{\arraystretch}{0.85}
\begin{tabular}{llccc}
\toprule
\textbf{$T'$} & \textbf{Policy} & {\textbf{Time [s]} $\downarrow$} & {\textbf{Unique Valid} $\uparrow$} & {$\hat{S}(\mathcal{G})$ $\uparrow$} \\
\midrule
\multirow{1}{*}{1000}  & Full  & $43.51 \pm 0.07$  & $\mathbf{2.18 \pm 1.02}$  & ${--}$ \\
\midrule
\multirow{9}{*}{5}  & Interval 1  & $20.03 \pm 0.04$  & $0.88 \pm 0.46$  & $36.77$ \\
  & Interval 2  & $19.99 \pm 0.03$  & $1.02 \pm 0.22$  & $329.40$ \\
  & Interval 3  & $19.98 \pm 0.01$  & $0.72 \pm 0.26$  & $276.90$ \\
  & Interval 4  & $19.98 \pm 0.02$  & $0.64 \pm 0.28$  & $155.30$ \\
  & Interval 5  & $19.98 \pm 0.02$  & $0.60 \pm 0.30$  & $21.21$ \\
  & Top-$V$  & $19.98 \pm 0.01$  & $0.74 \pm 0.32$  & $234.60$ \\
  & Top-$\Delta V$  & $19.98 \pm 0.01$  & $2.42 \pm 0.82$  & $395.30$ \\
  & Uniform  & $19.98 \pm 0.01$  & $1.26 \pm 0.34$  & $384.60$ \\
  & VISTA  & $19.98 \pm 0.01$  & $\mathbf{2.72 \pm 0.76}$  & $\mathbf{486.00}$ \\
\midrule
\multirow{9}{*}{10}  & Interval 1  & $20.17 \pm 0.03$  & $2.06 \pm 0.74$  & $67.63$ \\
  & Interval 2  & $20.09 \pm 0.02$  & $1.04 \pm 0.28$  & $328.90$ \\
  & Interval 3  & $20.09 \pm 0.02$  & $0.78 \pm 0.24$  & $278.40$ \\
  & Interval 4  & $20.09 \pm 0.03$  & $0.66 \pm 0.30$  & $154.40$ \\
  & Interval 5  & $20.09 \pm 0.02$  & $0.56 \pm 0.28$  & $21.10$ \\
  & Top-$V$  & $20.10 \pm 0.01$  & $0.76 \pm 0.38$  & $233.50$ \\
  & Top-$\Delta V$  & $20.10 \pm 0.01$  & $3.02 \pm 1.38$  & $441.30$ \\
  & Uniform  & $20.09 \pm 0.02$  & $2.02 \pm 0.46$  & $439.30$ \\
  & VISTA  & $20.09 \pm 0.02$  & $\mathbf{3.68 \pm 1.02}$  & $\mathbf{494.20}$ \\
\midrule
\multirow{9}{*}{15}  & Interval 1  & $20.29 \pm 0.02$  & $2.34 \pm 0.94$  & $102.70$ \\
  & Interval 2  & $20.23 \pm 0.02$  & $0.98 \pm 0.28$  & $329.50$ \\
  & Interval 3  & $20.22 \pm 0.01$  & $0.86 \pm 0.36$  & $278.10$ \\
  & Interval 4  & $20.22 \pm 0.03$  & $0.70 \pm 0.32$  & $153.00$ \\
  & Interval 5  & $20.22 \pm 0.03$  & $0.60 \pm 0.26$  & $20.99$ \\
  & Top-$V$  & $20.22 \pm 0.01$  & $0.80 \pm 0.36$  & $232.00$ \\
  & Top-$\Delta V$  & $20.21 \pm 0.01$  & $3.86 \pm 1.22$  & $446.10$ \\
  & Uniform  & $20.21 \pm 0.03$  & $3.08 \pm 0.64$  & $452.50$ \\
  & VISTA  & $20.21 \pm 0.02$  & $\mathbf{4.28 \pm 1.36}$  & $\mathbf{495.20}$ \\
\midrule
\multirow{9}{*}{20}  & Interval 1  & $20.41 \pm 0.03$  & $2.32 \pm 1.00$  & $134.50$ \\
  & Interval 2  & $20.34 \pm 0.02$  & $1.00 \pm 0.48$  & $336.80$ \\
  & Interval 3  & $20.34 \pm 0.03$  & $0.80 \pm 0.28$  & $277.50$ \\
  & Interval 4  & $20.34 \pm 0.03$  & $0.64 \pm 0.30$  & $152.40$ \\
  & Interval 5  & $20.34 \pm 0.02$  & $0.58 \pm 0.30$  & $20.89$ \\
  & Top-$V$  & $20.34 \pm 0.02$  & $0.82 \pm 0.40$  & $254.30$ \\
  & Top-$\Delta V$  & $20.34 \pm 0.02$  & $3.96 \pm 1.30$  & $457.30$ \\
  & Uniform  & $20.34 \pm 0.02$  & $3.42 \pm 0.72$  & $454.90$ \\
  & VISTA  & $20.32 \pm 0.02$  & $\mathbf{4.10 \pm 1.12}$  & $\mathbf{495.20}$ \\
\midrule
\multirow{9}{*}{25}  & Interval 1  & $20.46 \pm 0.02$  & $2.60 \pm 1.14$  & $174.10$ \\
  & Interval 2  & $20.45 \pm 0.02$  & $1.12 \pm 0.32$  & $339.20$ \\
  & Interval 3  & $20.45 \pm 0.01$  & $0.88 \pm 0.26$  & $272.80$ \\
  & Interval 4  & $20.45 \pm 0.02$  & $0.70 \pm 0.30$  & $151.50$ \\
  & Interval 5  & $20.44 \pm 0.02$  & $0.60 \pm 0.32$  & $20.78$ \\
  & Top-$V$  & $20.46 \pm 0.02$  & $0.80 \pm 0.38$  & $254.20$ \\
  & Top-$\Delta V$  & $20.45 \pm 0.02$  & $4.34 \pm 0.82$  & $457.90$ \\
  & Uniform  & $20.45 \pm 0.02$  & $3.90 \pm 0.56$  & $457.50$ \\
  & VISTA  & $20.43 \pm 0.02$  & $\mathbf{4.64 \pm 1.34}$  & $\mathbf{494.90}$ \\
\bottomrule
\end{tabular}
}
\end{table}

\begin{figure}
    \centering
    \includegraphics[width=\linewidth]{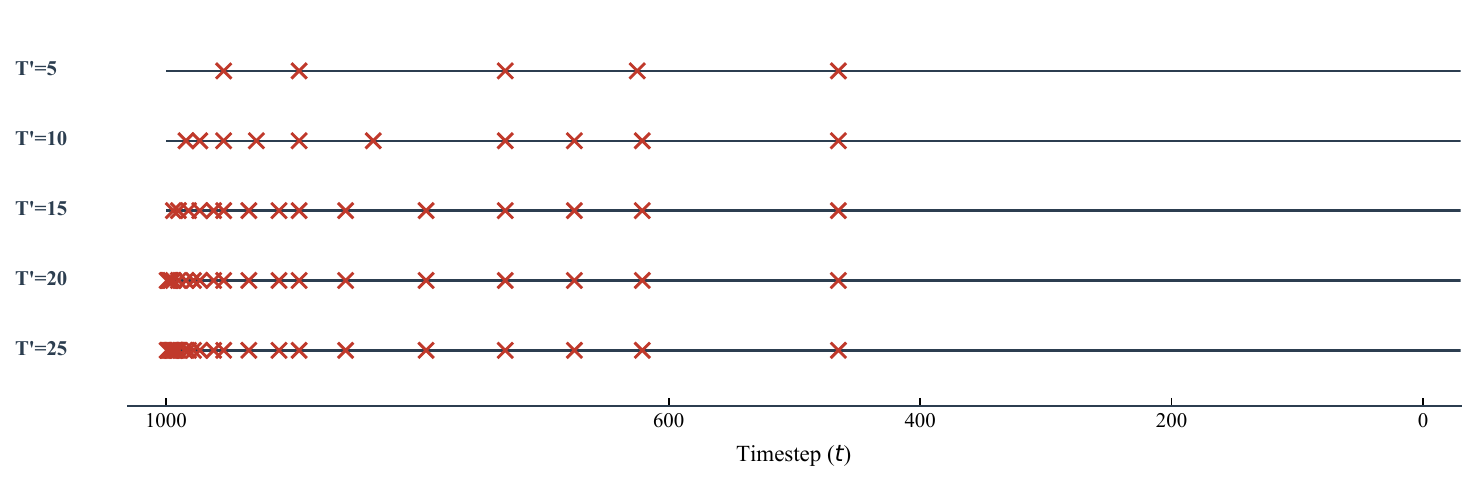}
    \caption{\method schedules $\mathcal{G}$ for varying $T'$ (CIFAR-10, SMC-base)}
    \label{fig:cifar10budget}
\end{figure}

\subsubsection{Details on Biology Design}\label{biologyappendix}

See Figure \ref{fig:dnabudget}, \ref{fig:proteinbudget}.

\begin{figure}
    \centering
    \includegraphics[width=\linewidth]{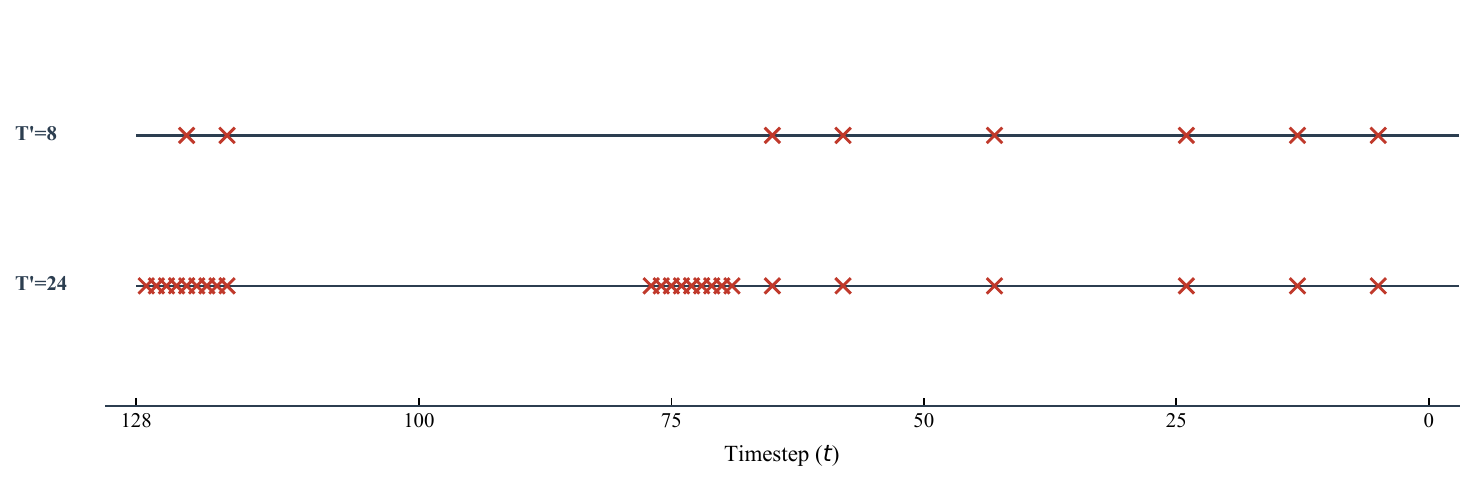}
    \caption{\method schedules $\mathcal{G}$ for varying $T'$ (DNA, SMC-grad)}
    \label{fig:dnabudget}
\end{figure}

\begin{figure}
    \centering
    \includegraphics[width=\linewidth]{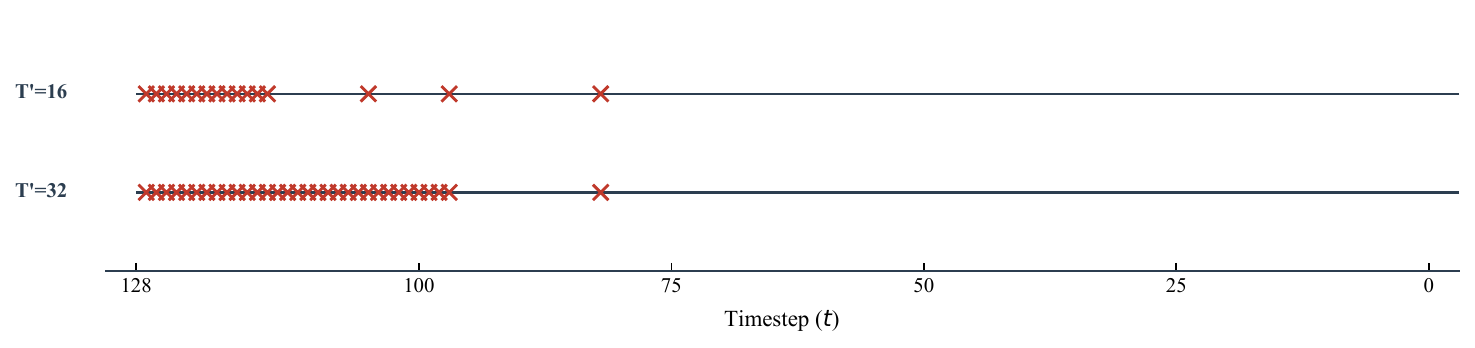}
    \caption{\method schedules $\mathcal{G}$ for varying $T'$ (Protein Fitness, MDM, SMC-base)}
    \label{fig:proteinbudget}
\end{figure}

\subsection{Experiments beyond masked diffusion models}

The weighted version of $\mathbf{TD}(\mathcal{G})$ is defined as 
\begin{equation}
    \mathbf{TD}(\mathcal{G}) := \sum_{t=0}^{T-1} \lambda(t) \| \pi_t^* -\pi_t^\mathcal{G} \|_1,
\end{equation}
where $\lambda(t)$ is nonnegative weight. We used uniform weight in masked diffusion models. However, in uniform discrete diffusion or continuous diffusion models, if the state takes the wrong path at the beginning of sampling, it could be corrected later. Thus, we should give bigger weight when $t$ is close to 0.

For the weighted version, by the proof of Theorem~\ref{UTDthm}, the following inequality holds:
\begin{align}
    \sum_{t=0}^{T-1} \lambda(t) \| \pi_t^* -\pi_t^\mathcal{G} \|_1 &\leq 2\sum_{t \in \mathcal{T}\setminus \mathcal{G}} \lambda(t)[1 - \frac{1}{R_{max}}\mathbb{E}_{x_{\lceil t \rceil} \sim \pi_{\lceil t \rceil}^*}[V_{\lceil t \rceil}(x_{\lceil t \rceil})]] \\
    &\leq 2\sum_{t=0}^{T-1} \lambda(t) - \frac{2}{R_{max}}\sum_{t\in \mathcal{T}\setminus \mathcal{G}} \lambda(t)\mathbb{E}_{x_{\lceil t \rceil} \sim \pi_{\lceil t \rceil}^*}[V_{\lceil t \rceil}(x_{\lceil t \rceil})].
\end{align}

Thus, we aim to maximize the surrogate score
\begin{equation}
    \hat{S}(\mathcal{G}) := \sum_{t \in \mathcal{T}\setminus \mathcal{G}} \lambda(t) \hat{V}_{\lceil t \rceil},
\end{equation}
in weighted VISTA, and the corresponding DP algorithm is modified into Alg \ref{alg:weightdp}. Further investigation on finding the optimal weight $\lambda(t)$ will be left to future work.

\begin{algorithm}
\caption{DP algorithm} \label{alg:weightdp}
\begin{algorithmic}[1]
    \State \text{Initialize: } $dp[s][i] \gets 0$ for all $s, i$
    \For{$s\gets T'-1, T', \cdots, T-1$}
        \State $dp[s][1] \gets (\sum_{u>s} \lambda(u))\hat{V}_T$
    \EndFor
    \For{$i\gets 2, 3, \cdots, T'$}
        \For{$s \gets T'-i, \cdots, T-i$}
            \For{$s' \gets s+1$ to $T-i+1$}
                \State $c \gets dp[s'][i-1] + (\sum_{s<u<s'} \lambda(u))\hat{V}_{s'}$
                \If{$c > dp[s][i]$}
                    \State $dp[s][i] \gets c$
                    \State $parent[s][i] \gets s'$
                \EndIf
            \EndFor
        \EndFor
    \EndFor
    \State $t_{T'} \gets \text{argmax}_s (dp[s][T'] + (\sum_{u<s} \lambda(u))\hat{V}_s)$
    \For{$j\gets T'-1, T'-2, \cdots, 1$}
        \State $t_j \gets parent[t_{j+1}][j+1]$
    \EndFor
    \State \Return $\mathcal{G} = \{t_1, t_2, \cdots, t_{T'}\}$
\end{algorithmic}
\end{algorithm}

\subsubsection{Protein Fitness (Uniform discrete diffusion)}

We use $\lambda(t) = 1- t/T$. We use reward model $r(x_0)$ as oracle that predicts fitness. We set $T=128, N=32, J=3, M=3, \alpha = 0.2$. A sample $x_0$ is considered ``valid'' if $r(x_0) \geq 0.5$. Two sequences with Hamming distance $\leq 1$ are treated as identical. Results are averaged over 20 independent runs, yielding 640 generated samples per policy. The warmup phase took 18.7\,s.
See Table \ref{tab:protein-udlm-smc-base}.

\begin{table}[t]
\centering
\small
\caption{Protein fitness generation, Uniform Discrete Diffusion (SMC-base, 32 particles $\times$ 20 runs): execution time, unique valid count, and $\hat{S}(\mathcal{G})$.}
\label{tab:protein-udlm-smc-base}
\resizebox{\linewidth}{!}{
\setlength{\tabcolsep}{6pt}
\renewcommand{\arraystretch}{1.15}
\begin{tabular}{llccc}
\toprule
\textbf{$T'$} & \textbf{Policy} & {\textbf{Time [s]} $\downarrow$} & {\textbf{Unique Valid} $\uparrow$} & {$\hat{S}(\mathcal{G})$ $\uparrow$} \\
\midrule
\multirow{1}{*}{128}  & Full  & $3.26$  & $\mathbf{3.15 \pm 2.37}$  & {--} \\
\midrule
\multirow{9}{*}{16}  & Interval 1  & $1.42$  & $4.45 \pm 3.19$  & $2.71\times10^{3}$ \\
  & Interval 2  & $1.41$  & $7.50 \pm 3.72$  & $6.05\times10^{3}$ \\
  & Interval 3  & $1.41$  & $6.10 \pm 3.93$  & $7.86\times10^{3}$ \\
  & Interval 4  & $1.41$  & $5.20 \pm 2.84$  & $3.53\times10^{3}$ \\
  & Interval 5  & $1.41$  & $2.20 \pm 1.20$  & $602.24$ \\
  & Top-$V$  & $1.42$  & $5.95 \pm 3.59$  & $1.19\times10^{4}$ \\
  & Top-$\Delta V$  & $1.41$  & $6.40 \pm 3.91$  & $1.35\times10^{4}$ \\
  & Uniform  & $1.41$  & $6.20 \pm 3.90$  & $1.10\times10^{4}$ \\
  & VISTA  & $1.42$  & $\mathbf{8.75 \pm 4.76}$  & $\mathbf{2.06\times10^{4}}$ \\
\midrule
\multirow{9}{*}{32}  & Interval 1  & $1.85$  & $8.90 \pm 5.31$  & $1.07\times10^{4}$ \\
  & Interval 2  & $1.86$  & $6.80 \pm 4.64$  & $6.02\times10^{3}$ \\
  & Interval 3  & $1.85$  & $6.65 \pm 4.32$  & $6.08\times10^{3}$ \\
  & Interval 4  & $1.85$  & $4.70 \pm 3.34$  & $4.35\times10^{3}$ \\
  & Interval 5  & $1.85$  & $2.95 \pm 2.50$  & $443.11$ \\
  & Top-$V$  & $1.85$  & $5.60 \pm 2.95$  & $1.06\times10^{4}$ \\
  & Top-$\Delta V$  & $1.85$  & $6.05 \pm 3.52$  & $1.08\times10^{4}$ \\
  & Uniform  & $1.84$  & $5.45 \pm 2.14$  & $8.88\times10^{3}$ \\
  & VISTA  & $1.85$  & $\mathbf{9.95 \pm 5.30}$  & $\mathbf{1.97\times10^{4}}$ \\
\bottomrule
\end{tabular}
}
\end{table}

\subsubsection{Continuous MNIST (Continuous diffusion)}

We use $\lambda(t) = (1- t/T)^{20}$. We use reward model $r(x_0) = \log p(c|x_0)$ where $c$ indicates the condition `image is 4'. We set $T=100, N=20, J=1, M=3, \alpha = 2.0$. A sample $x_0$ is considered ``valid'' if $r(x_0) > -0.1$. If the binarized images have Jaccard similarity $\geq 0.85$, then we regard them as duplicates. Results are averaged over 30 independent runs, yielding 600 generated samples per policy. The warmup phase took 41.0 s. See Table \ref{tab:mnisttds-time-valid-score}.

\begin{table}[t]
\centering
\small
\caption{MNIST (TDS, reproduce full/sparse): execution time, unique valid count, and $\hat{S}(\mathcal{G})$.}
\label{tab:mnisttds-time-valid-score}
\resizebox{\linewidth}{!}{
\setlength{\tabcolsep}{6pt}
\renewcommand{\arraystretch}{1.15}
\begin{tabular}{llccc}
\toprule
\textbf{$T'$} & \textbf{Policy} & {\textbf{Time [s]} $\downarrow$} & {\textbf{Unique Valid} $\uparrow$} & {$\hat{S}(\mathcal{G})$ $\uparrow$} \\
\midrule
\multirow{1}{*}{100}  & Full  & $15.47 \pm 0.08$  & $\mathbf{10.33 \pm 3.45}$  & {--} \\
\midrule
\multirow{9}{*}{6}  & Interval 1  & $4.29 \pm 0.20$  & $1.93 \pm 1.44$  & $0.06$ \\
  & Interval 2  & $4.24 \pm 0.03$  & $1.60 \pm 1.25$  & $0.36$ \\
  & Interval 3  & $4.24 \pm 0.02$  & $7.77 \pm 4.64$  & $3.40$ \\
  & Interval 4  & $4.25 \pm 0.03$  & $12.87 \pm 6.77$  & $5.08$ \\
  & Interval 5  & $4.15 \pm 0.03$  & $1.73 \pm 1.28$  & $0.11$ \\
  & Top-$V$  & $4.26 \pm 0.03$  & $11.73 \pm 5.91$  & $4.92$ \\
  & Top-$\Delta V$  & $4.26 \pm 0.03$  & $14.13 \pm 4.91$  & $4.34$ \\
  & Uniform  & $4.16 \pm 0.02$  & $14.10 \pm 6.21$  & $4.12$ \\
  & VISTA  & $4.26 \pm 0.03$  & $\mathbf{16.83 \pm 5.02}$  & $\mathbf{5.11}$ \\
\midrule
\multirow{9}{*}{12}  & Interval 1  & $4.95 \pm 0.37$  & $1.93 \pm 1.01$  & $0.09$ \\
  & Interval 2  & $4.86 \pm 0.05$  & $1.30 \pm 1.09$  & $1.21$ \\
  & Interval 3  & $4.85 \pm 0.04$  & $13.43 \pm 4.90$  & $4.42$ \\
  & Interval 4  & $4.86 \pm 0.05$  & $14.97 \pm 5.57$  & $5.06$ \\
  & Interval 5  & $4.77 \pm 0.05$  & $1.87 \pm 1.38$  & $0.03$ \\
  & Top-$V$  & $4.86 \pm 0.37$  & $8.20 \pm 4.21$  & $3.07$ \\
  & Top-$\Delta V$  & $4.90 \pm 0.42$  & $14.03 \pm 5.40$  & $4.67$ \\
  & Uniform  & $4.75 \pm 0.02$  & $15.73 \pm 5.72$  & $3.96$ \\
  & VISTA  & $5.31 \pm 0.97$  & $\mathbf{17.13 \pm 3.80}$  & $\mathbf{5.11}$ \\
\bottomrule
\end{tabular}
}
\end{table}

\subsection{Sensitivity Analysis}

Throughout the paper, we use $M$ warmup simulations to estimate the value function for schedule optimization. Although this is a relatively small number of simulations, we find that it is sufficient to recover the high-level structure of the optimized guidance schedule. To verify this, we vary $M$ and visualize the resulting schedules in Figures~\ref{fig:mnistbaseM}, \ref{fig:mnistgradM}, \ref{fig:toxicgradM}, \ref{fig:cifarbaseM}, \ref{fig:dnagradM}, \ref{fig:proteinmdmM}, \ref{fig:proteinudmM}, and \ref{fig:mnisttdsM}. Across all tasks, the value-critical guidance regions remain largely unchanged once $M=3$ ($M=5$ in CIFAR-10) is reached, indicating that only a small number of warmup simulations is required in practice.


\begin{figure}
    \centering
    \includegraphics[width=\linewidth]{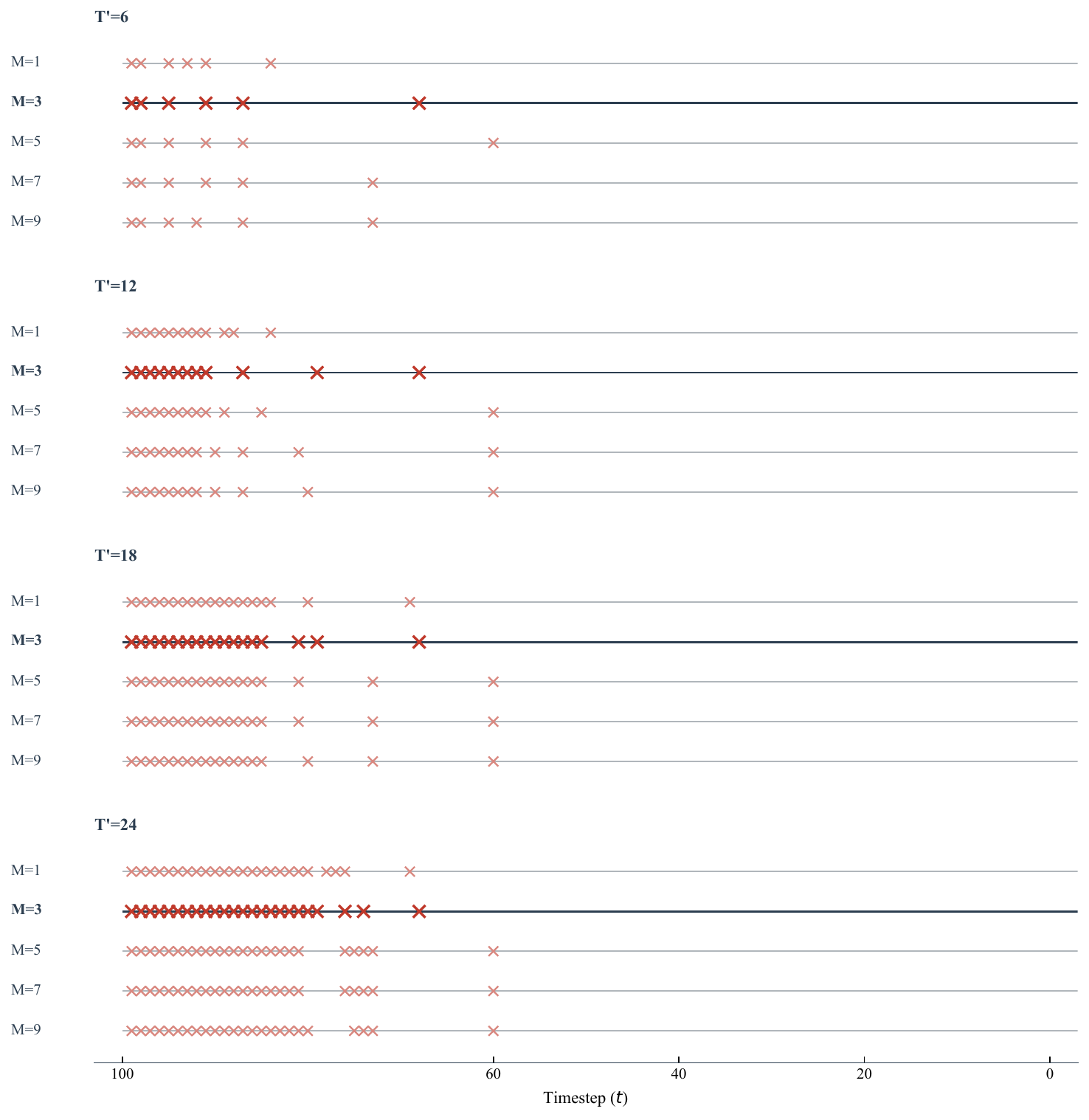}
    \caption{Optimized guidance schedules $\mathcal{G}$ for varying $M$ (Binary MNIST, SMC-base)}
    \label{fig:mnistbaseM}
\end{figure}

\begin{figure}
    \centering
    \includegraphics[width=\linewidth]{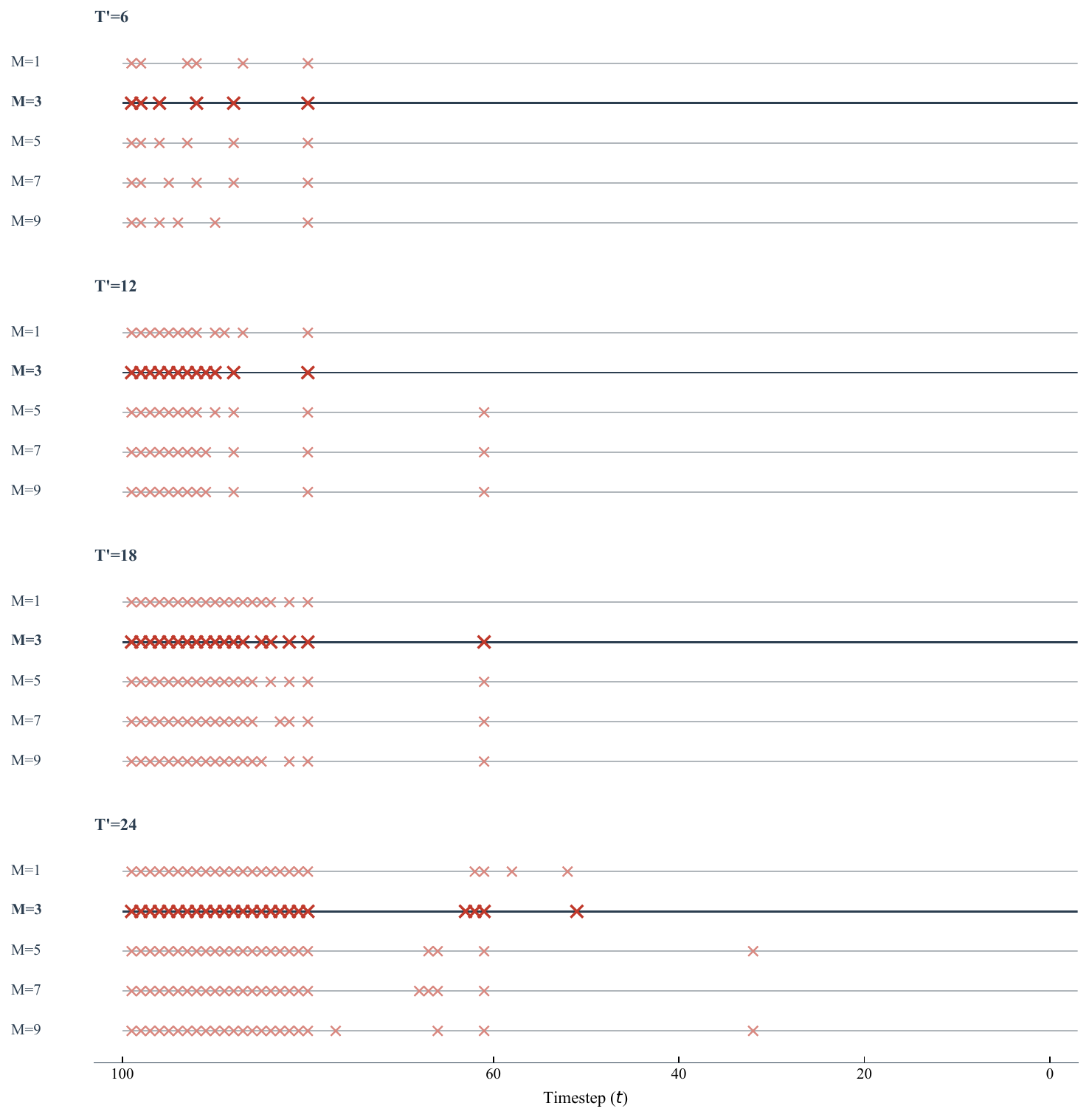}
    \caption{Optimized guidance schedules $\mathcal{G}$ for varying $M$ (Binary MNIST, SMC-grad)}
    \label{fig:mnistgradM}
\end{figure}

\begin{figure}
    \centering
    \includegraphics[width=\linewidth]{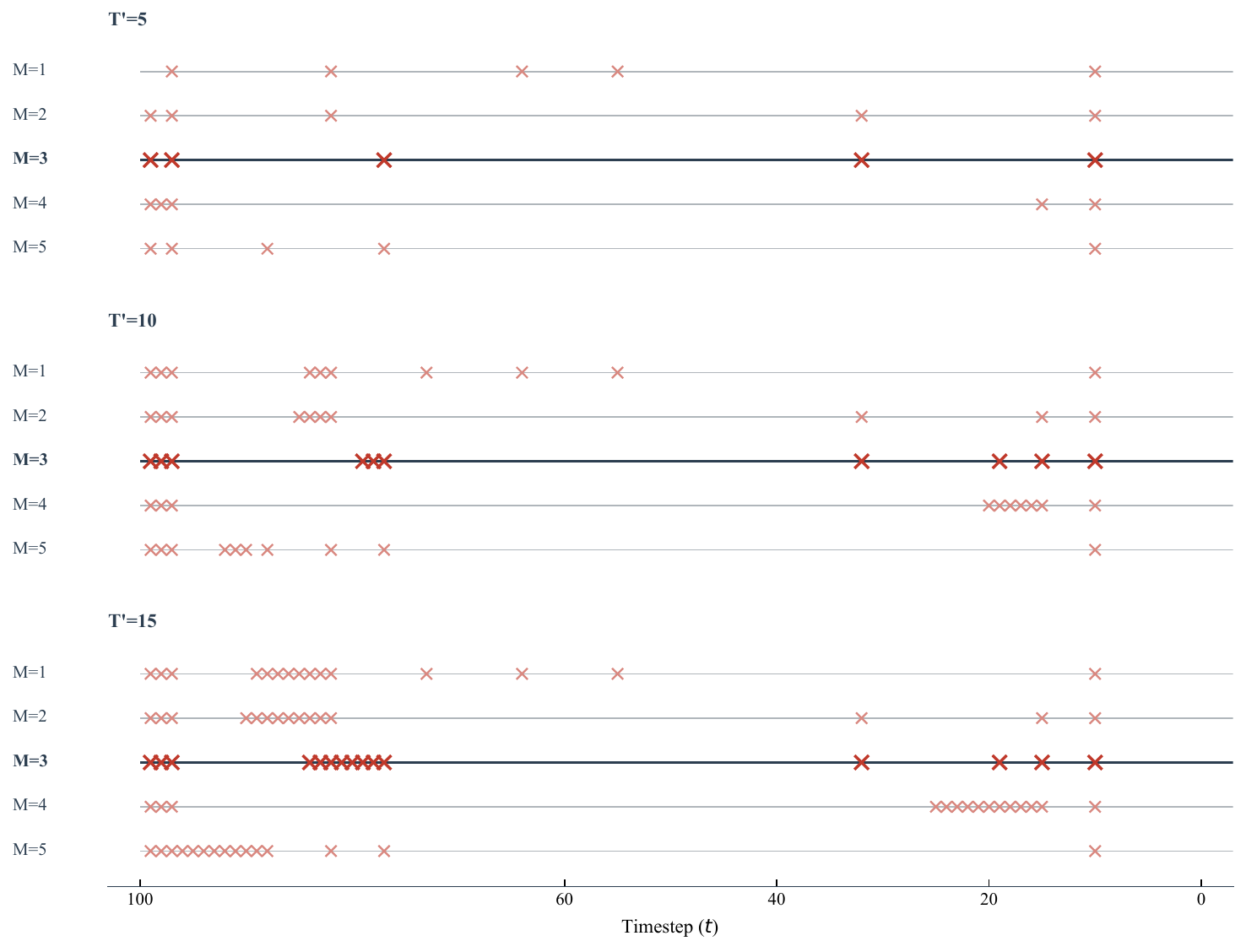}
    \caption{Optimized guidance schedules $\mathcal{G}$ for varying $M$ (Toxic Text, SMC-grad)}
    \label{fig:toxicgradM}
\end{figure}

\begin{figure}
    \centering
    \includegraphics[width=\linewidth]{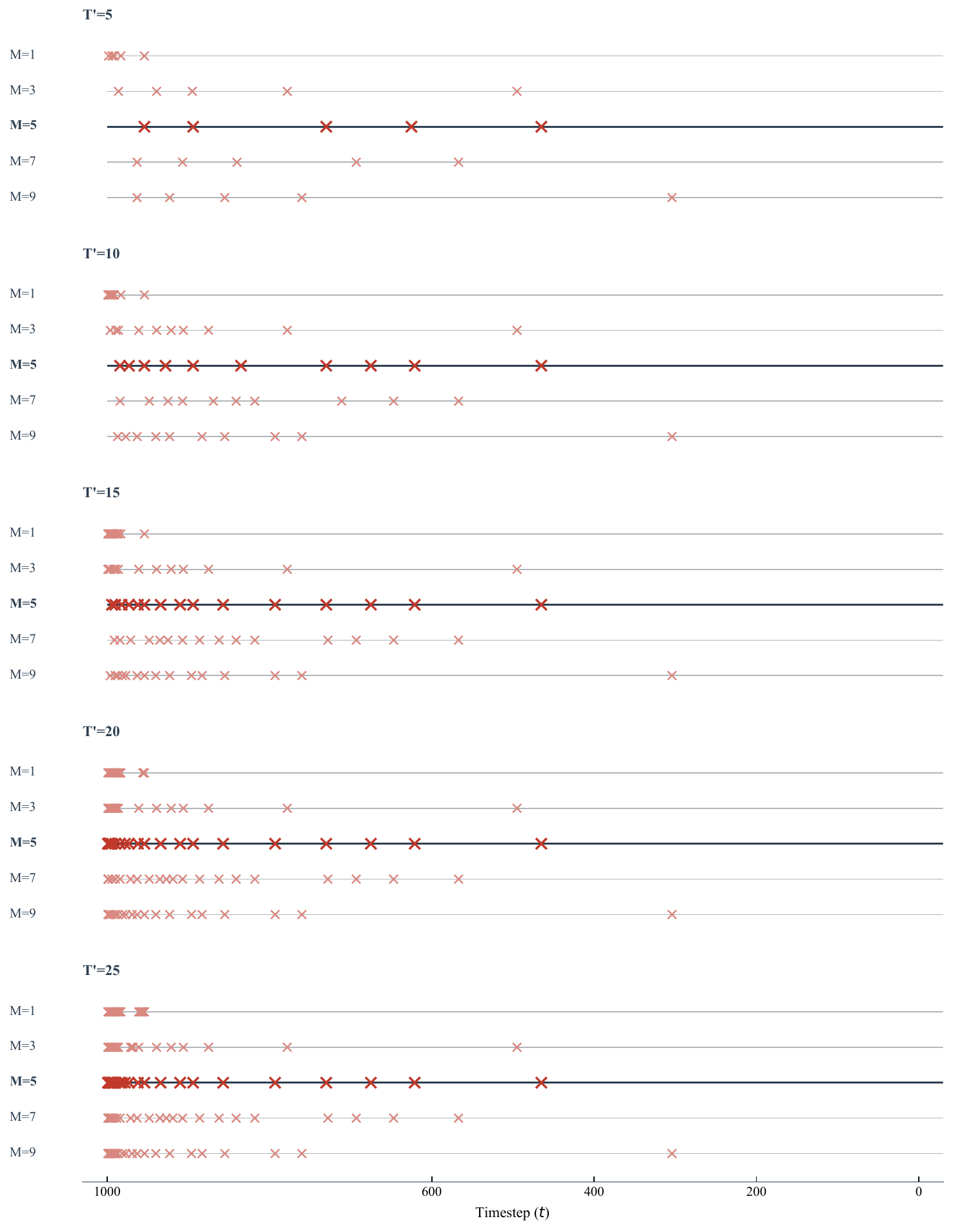}
    \caption{Optimized guidance schedules $\mathcal{G}$ for varying $M$ (CIFAR-10, SMC-base)}
    \label{fig:cifarbaseM}
\end{figure}

\begin{figure}
    \centering
    \includegraphics[width=\linewidth]{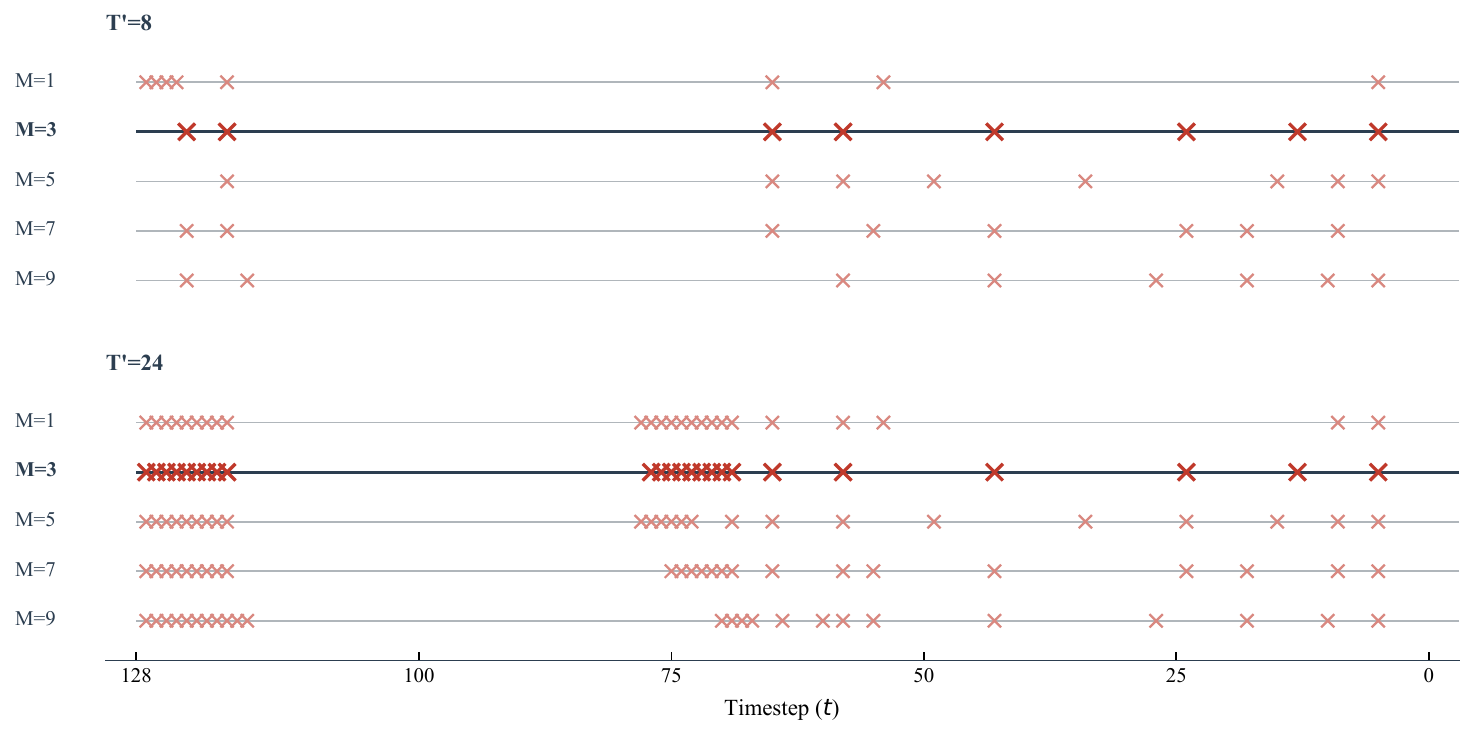}
    \caption{Optimized guidance schedules $\mathcal{G}$ for varying $M$ (DNA, SMC-grad)}
    \label{fig:dnagradM}
\end{figure}

\begin{figure}
    \centering
    \includegraphics[width=\linewidth]{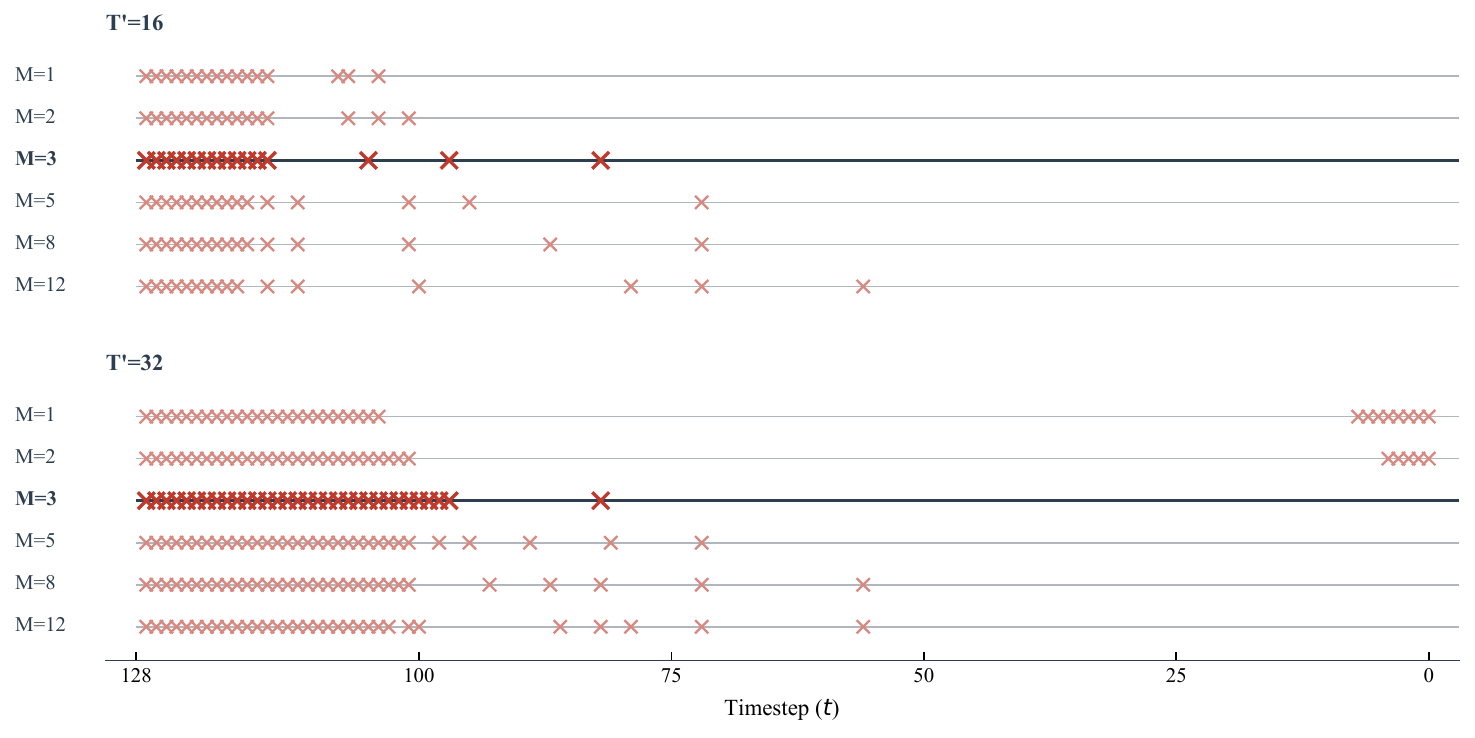}
    \caption{Optimized guidance schedules $\mathcal{G}$ for varying $M$ (Protein Fitness, SMC-base, MDM)}
    \label{fig:proteinmdmM}
\end{figure}

\begin{figure}
    \centering
    \includegraphics[width=\linewidth]{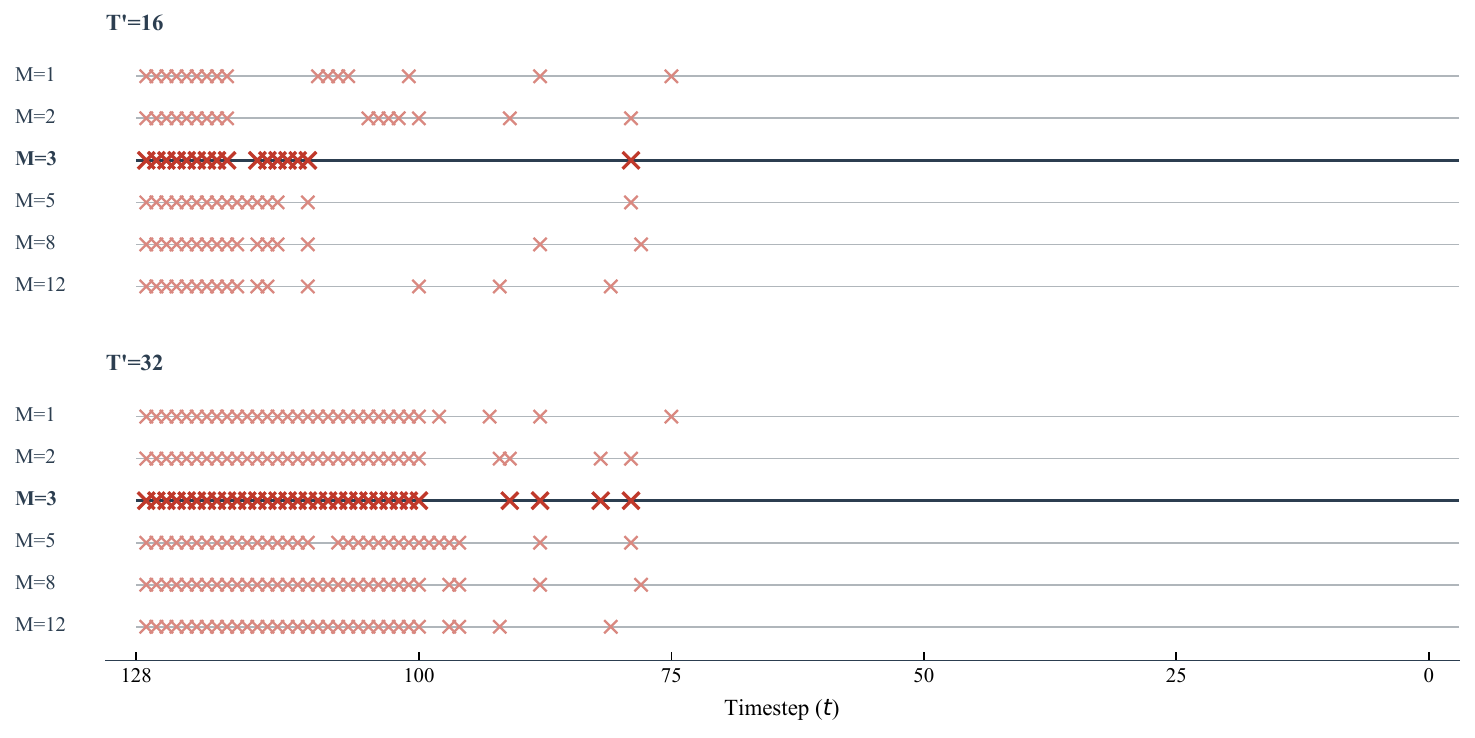}
    \caption{Optimized guidance schedules $\mathcal{G}$ for varying $M$ (Protein Fitness, SMC-base, UDM)}
    \label{fig:proteinudmM}
\end{figure}

\begin{figure}
    \centering
    \includegraphics[width=\linewidth]{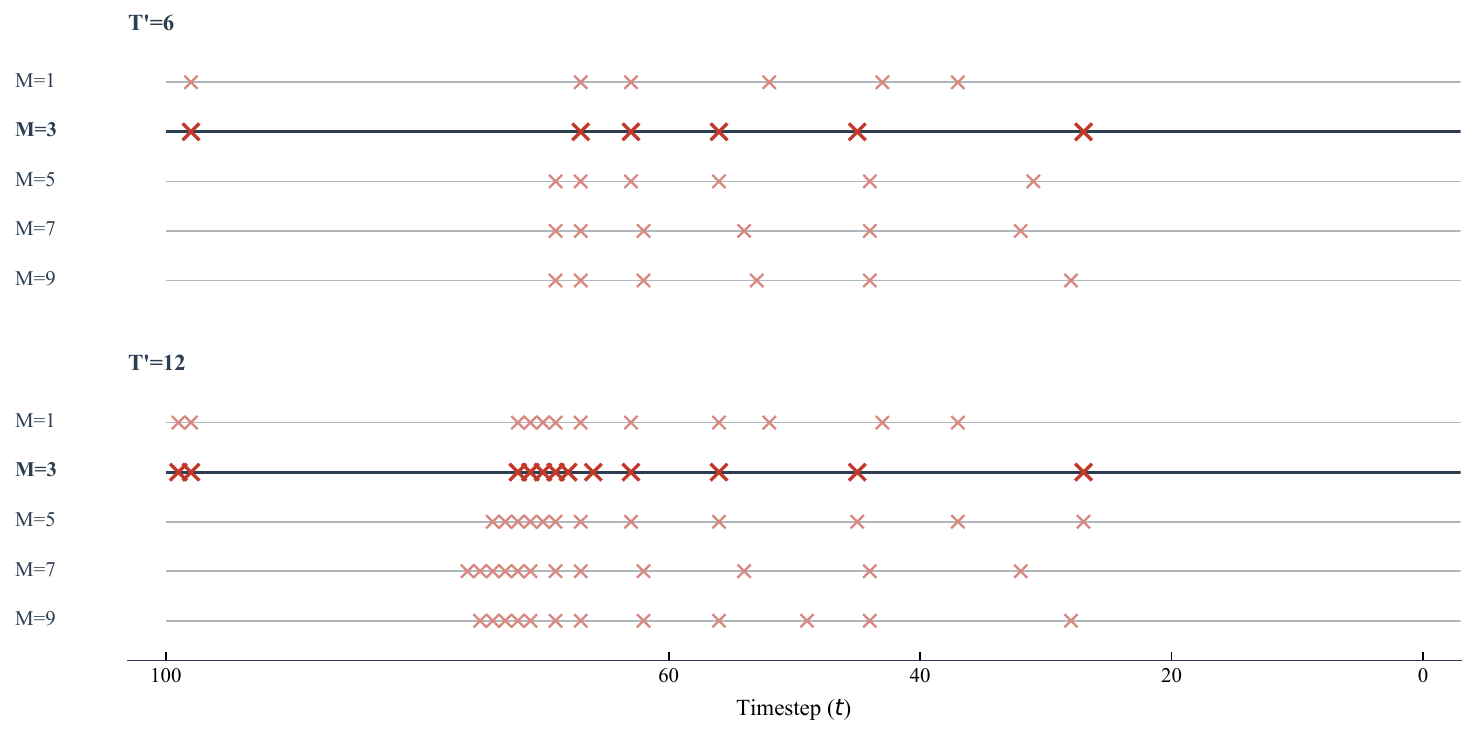}
    \caption{Optimized guidance schedules $\mathcal{G}$ for varying $M$ (MNIST, TDS)}
    \label{fig:mnisttdsM}
\end{figure}

\section{Societal Impacts}

\paragraph{Positive impacts.} \method reduces the cost of SMC-based reward guidance. This can make inference-time alignment more accessible to researchers with limited hardware and can reduce the energy cost of repeated reward-guided generation.

\paragraph{Potential negative impacts.} More efficient control of generative models could be misused to generate harmful or manipulative content at lower cost. The toxic-text experiment is included only as a controlled benchmark for testing reward steering; it is not an endorsement of generating toxic content.

\paragraph{Mitigation.} The main risk depends on the reward model and application. We recommend using \method with ethically aligned reward functions, safety filters, monitoring, and application-specific release controls. The scheduler itself is neutral: it decides when a given reward should be applied, but it does not define the reward.

\end{document}